%% file: main.tex
\documentclass{article} %

\usepackage{etoolbox}
\usepackage{comment}
\newtoggle{iclr}
\togglefalse{iclr}

\newcommand{\iclr}[1]{\iftoggle{iclr}{#1}{}}
\newcommand{\arxiv}[1]{\iftoggle{iclr}{}{#1}}
\newcommand{\loose}{\looseness=-1}

\iclr{
\usepackage[dvipsnames]{xcolor} %

\usepackage{iclr2026_conference}
\usepackage{times}
}

\input{arxiv_style}

\usepackage[utf8]{inputenc} %
\usepackage[T1]{fontenc}    %
\usepackage{url}            %
\usepackage{booktabs}       %
\usepackage{amsfonts}       %
\usepackage{nicefrac}       %
\usepackage{microtype}      %
\usepackage{makecell}
\usepackage{enumitem}
\usepackage{breakcites}
\usepackage{mathrsfs}
\usepackage{wrapfig}

\usepackage{algorithm}
\usepackage{verbatim}
\usepackage[noend]{algpseudocode}
\newcommand{\multiline}[1]{\parbox[t]{\dimexpr\linewidth-\algorithmicindent}{#1}}
\makeatletter
\let\OldStatex\Statex
\renewcommand{\Statex}[1][3]{%
  \setlength\@tempdima{\algorithmicindent}%
  \OldStatex\hskip\dimexpr#1\@tempdima\relax}
\makeatother

\usepackage{multicol}
\usepackage{colortbl}
\usepackage{setspace}
\usepackage{transparent}
\usepackage{upgreek}
\usepackage{inconsolata}
\usepackage[scaled=.86]{helvet}
\usepackage{xspace}

\iclr{
\setlist[enumerate]{leftmargin=*}
\setlist[itemize]{leftmargin=*}
}

\usepackage{graphicx}
\usepackage{adjustbox}
\usepackage{mwe}
\iclr{\usepackage{subfigure}}

\input{dylan}

\input{macros}

\input{widebar}

\input{math_commands.tex}

\usepackage[suppress]{color-edits}
 \addauthor{df}{ForestGreen}
 \addauthor{ak}{BurntOrange}
 \addauthor{jta}{red}
 \addauthor{jt}{cyan}
 \newcommand{\dfc}[1]{}
 \newcommand{\jta}[1]{}
 \newcommand{\akc}[1]{}
 \newcommand{\jt}[1]{}

\let\oldparagraph\paragraph

\renewcommand{\paragraph}[1]{\oldparagraph{#1.}}

\newcommand{\textbfc}[1]{\textbf{\textcolor{blue!40!black}{#1}}}

\usepackage{pifont}

\iclr{
\title{
\mbox{Representation-Based Exploration for Language} Models:
 From Test-Time to Post-Training
}
}
\arxiv{
\title{
Representation-Based Exploration for Language Models:\\ From Test-Time to Post-Training
}
}

\iclr{
 \author{Jens Tuyls\textsuperscript{1,}\thanks{Work partially completed during an internship at Microsoft Research.} \quad Dylan J. Foster\textsuperscript{2} \quad Akshay Krishnamurthy\textsuperscript{2} \quad Jordan T. Ash\textsuperscript{2} \\
\textsuperscript{1}Princeton University \quad \textsuperscript{2}Microsoft Research NYC \\
\texttt{jtuyls@cs.princeton.edu}, \texttt{\{dylanfoster,akshaykr,ash.jordan\}@microsoft.com}
}
}

\arxiv{
 \author{Jens Tuyls$^{1,}$\thanks{Work partially completed during an internship at Microsoft Research.} \quad Dylan J. Foster$^{2}$ \quad Akshay Krishnamurthy$^{2}$ \quad Jordan T. Ash$^{2}$ \vspace{0.5em}\\
 {\normalsize $^{1}$Princeton University \quad $^{2}$Microsoft Research NYC}
 }

\date{}
}

\iclr{\iclrfinalcopy} %
\begin{document}

\maketitle

\begin{abstract}
\input{sections/abstract}

\end{abstract}

\section{Introduction}
\label{sec:intro}

\input{sections/introduction}

\section{Representation-Based Exploration: Inference-Time and RL}
\label{sec:methods}

\input{sections/main}

\section{Exploration for RL Post-Training}
\label{sec:rl}

\input{sections/rl}

\section{Discussion}
\label{sec:conclusion}
\input{sections/conclusion}

\section*{Acknowledgements}

\input{sections/acknowledgements}
\label{sec:acknowledgements}

\section*{Reproducibility Statement}
\label{sec:reproducibility}

\input{sections/reproducibility}

\clearpage
\iclr{
\bibliographystyle{iclr2026_conference}
}
\bibliography{refs}

\clearpage

\appendix

\section{Additional Related Work}
\label{sec:additional_related}

\input{sections/related}

\section{Details for Inference-Time Experiments (\cref{sec:inference_time_exploration})}
\label{sec:appendix_inference_time}
\subsection{Details from \cref{sec:coresets}}
\label{sec:appendix_coresets}

\input{sections/details_for_coreset}

\subsection{Details from \cref{sec:guided}}
\label{sec:appendix_guided}

\input{sections/details_for_token_level}

\section{Details for RL Post-Training Experiments (\cref{sec:rl})}
\label{sec:appendix_rl}

\input{sections/details_for_rl}

\iclr{
\section{Protein Sequence Generation}
\label{sec:appendix_protein_sequence_generation}
\input{sections/protein}
}
\section{Beyond Sharpening: Full Results}
\label{sec:appendix_beyond_sharpening}
\input{sections/beyond_sharpening}

\end{document}

%% file: arxiv_style.tex
\arxiv{
\usepackage[letterpaper, left=1in, right=1in, top=1in,bottom=1in]{geometry}
  \usepackage{parskip}
  \usepackage[colorlinks=true, linkcolor=blue!70!black, citecolor=blue!70!black,urlcolor=black,breaklinks=true]{hyperref}
  \usepackage[dvipsnames]{xcolor}
}

\iclr{
  \usepackage{parskip}
    \usepackage[colorlinks=true, linkcolor=blue!60!black, citecolor=blue!60!black,urlcolor=black,breaklinks=true]{hyperref}
}

\PassOptionsToPackage{hypertexnames=false}{hyperref}  %

\usepackage{amsmath}
\usepackage{microtype}
\usepackage{hhline}

\usepackage{amsthm}
\usepackage{mathtools}
\usepackage{bbm}
\usepackage{amsfonts}
\usepackage{amssymb}
\usepackage[nameinlink,capitalize]{cleveref}

\makeatletter
\newcommand{\neutralize}[1]{\expandafter\let\csname c@#1\endcsname\count@}
\makeatother

\usepackage{algorithm}

\arxiv{
\usepackage{natbib}
\bibliographystyle{plainnat}
\bibpunct{(}{)}{;}{a}{,}{,}
}

\usepackage{xpatch}

\usepackage{thm-restate}
\declaretheorem[name=Theorem,parent=section]{theorem}
\declaretheorem[name=Lemma,parent=section]{lemma}
\declaretheorem[name=Assumption, parent=section]{assumption}
\declaretheorem[name=Definition, parent=section]{definition}
\declaretheorem[name=Condition, parent=section]{condition}
\declaretheorem[name=Corollary, parent=section]{corollary}

\declaretheorem[name=Remark, parent=section]{remark}
\declaretheorem[name=Proposition, parent=section]{proposition}

\newcommand{\pfref}[1]{Proof of \cref{#1}}

\renewcommand{\eqref}[1]{\texorpdfstring{\hyperref[#1]{(\ref*{#1})}}{(\ref*{#1})}}
\crefformat{equation}{#2Eq.\,(#1)#3}
\Crefformat{equation}{#2Eq.\,(#1)#3}
\Crefformat{figure}{#2Figure~#1#3}
\Crefformat{assumption}{#2Assumption~#1#3}
\Crefname{assumption}{Assumption}{Assumptions}
\crefname{fact}{Fact}{Facts}
\Crefformat{figure}{#2Figure #1#3}
\Crefformat{assumption}{#2Assumption #1#3}

\usepackage{crossreftools}
\makeatletter
\renewenvironment{proof}[1][Proof]%
{%
  \par\noindent{\bfseries\upshape {#1.}\ }%
}%
{\qed\newline}
\makeatother

\xpatchcmd{\proof}{\itshape}{\normalfont\proofnameformat}{}{}
\newcommand{\proofnameformat}{\bfseries}

\usepackage[most]{tcolorbox}

\tcbset {
  base/.style={
    arc=0mm, 
    bottomtitle=0.5mm,
    boxrule=0mm,
    colbacktitle=!10!white, 
    coltitle=black, 
    fonttitle=\bfseries, 
    left=2.5mm,
    leftrule=1mm,
    right=3.5mm,
    title={#1},
    toptitle=0.75mm, 
  }
}

\newtcolorbox{mainbox}[1][]{
  colframe=blue!40!black,
  colback=blue!2!white,
  enhanced,
  attach boxed title to top text left={yshift=-2mm},
  boxed title style={size=small,colframe=blue!40!black,colback=blue!40!black} 
  title={\textbf{#1}}
}

\newtcolorbox{minbox}[1]{
  colframe=blue!10!black,
  colbacktitle=blue!50!black!30!white,
  colback=blue!2!white,
  enhanced,
  fonttitle=\bfseries,
}

%% file: dylan.tex
\DeclarePairedDelimiter{\brk}{[}{]}
\DeclarePairedDelimiter{\crl}{\{}{\}}

\DeclareMathOperator{\En}{\mathbb{E}}

\DeclareMathOperator*{\argmax}{arg\,max}

\def\ddefloop#1{\ifx\ddefloop#1\else\ddef{#1}\expandafter\ddefloop\fi}
\def\ddef#1{\expandafter\def\csname bb#1\endcsname{\ensuremath{\mathbb{#1}}}}
\ddefloop ABCDEFGHIJKLMNOPQRSTUVWXYZ\ddefloop
\def\ddefloop#1{\ifx\ddefloop#1\else\ddef{#1}\expandafter\ddefloop\fi}
\def\ddef#1{\expandafter\def\csname b#1\endcsname{\ensuremath{\mathbf{#1}}}}
\ddefloop ABCDEFGHIJKLMNOPQRSTUVWXYZ\ddefloop
\def\ddef#1{\expandafter\def\csname sf#1\endcsname{\ensuremath{\mathsf{#1}}}}
\ddefloop ABCDEFGHIJKLMNOPQRSTUVWXYZ\ddefloop
\def\ddef#1{\expandafter\def\csname c#1\endcsname{\ensuremath{\mathcal{#1}}}}
\ddefloop ABCDEFGHIJKLMNOPQRSTUVWXYZ\ddefloop
\def\ddef#1{\expandafter\def\csname h#1\endcsname{\ensuremath{\widehat{#1}}}}
\ddefloop ABCDEFGHIJKLMNOPQRSTUVWXYZ\ddefloop
\def\ddef#1{\expandafter\def\csname hc#1\endcsname{\ensuremath{\widehat{\mathcal{#1}}}}}
\ddefloop ABCDEFGHIJKLMNOPQRSTUVWXYZ\ddefloop
\def\ddef#1{\expandafter\def\csname t#1\endcsname{\ensuremath{\widetilde{#1}}}}
\ddefloop ABCDEFGHIJKLMNOPQRSTUVWXYZ\ddefloop
\def\ddef#1{\expandafter\def\csname tc#1\endcsname{\ensuremath{\widetilde{\mathcal{#1}}}}}
\ddefloop ABCDEFGHIJKLMNOPQRSTUVWXYZ\ddefloop
\def\ddefloop#1{\ifx\ddefloop#1\else\ddef{#1}\expandafter\ddefloop\fi}
\def\ddef#1{\expandafter\def\csname scr#1\endcsname{\ensuremath{\mathscr{#1}}}}
\ddefloop ABCDEFGHIJKLMNOPQRSTUVWXYZ\ddefloop

\newcommand{\ldef}{\vcentcolon=}

%% file: macros.tex
\newcommand{\repexp}{\texttt{RepExp}\xspace}

\makeatletter
\newcommand{\passatk}[1][N]{pass@$#1$\xspace}
\makeatother

\renewcommand{\div}{\texttt{div}}

\newcommand{\pirefh}[1][h]{\pi_{h,\texttt{ref}}}

\newcommand{\rstar}{r^{\star}}

\newcommand{\xpo}{\texttt{XPO}\xspace}

\newcommand{\M}[1]{^{{\scriptscriptstyle M}}}  %

\newcommand{\trn}{\top}

\newcommand{\bigoh}{O}

\def\multiset#1#2{\ensuremath{\left(\kern-.3em\left(\genfrac{}{}{0pt}{}{#1}{#2}\right)\kern-.3em\right)}}

\newcommand{\iidsim}{\overset{\textrm{i.i.d.}}{\sim}}

\newcommand{\Alg}{\mathtt{Alg}}

\newcommand{\inner}[2]{\langle #1,#2\rangle}

%% file: math_commands.tex
\usepackage{amsmath,amsfonts,bm}

\def\eqref#1{equation~\ref{#1}}

\def\1{\bm{1}}

\DeclareMathAlphabet{\mathsfit}{\encodingdefault}{\sfdefault}{m}{sl}
\SetMathAlphabet{\mathsfit}{bold}{\encodingdefault}{\sfdefault}{bx}{n}

%% file: sections/abstract.tex
Reinforcement learning (RL) promises to expand the capabilities of language models, but it is unclear if current RL techniques promote the discovery of novel behaviors, or simply sharpen those already present in the base model.
In this paper, we investigate the value of deliberate exploration---explicitly incentivizing the model to discover novel and diverse behaviors---and aim to understand how the knowledge in pre-trained models can guide this search. Our main finding is that exploration with a simple, principled, 
\textbfc{representation-based} bonus derived from the pre-trained language model's hidden states significantly improves diversity and pass@k rates---both for post-training, and in a novel inference-time scaling setting we introduce. \loose

\begin{enumerate}
\item For inference-time, exploration with representation-based diversity
 improves efficiency, consistently improving pass@k rates across a variety of models and reasoning tasks. For example, for \texttt{Qwen-2.5-14b-Instruct} we obtain over 50\% improvement in verifier efficiency on almost all tasks.\loose
\item For post-training, we show that integrating this exploration strategy into an RL pipeline improves reasoning performance over that of the initial model and over standard RL post-training. For example, on \texttt{AIME 2024}, our post-trained \texttt{Qwen-2.5-7b-Instruct}'s pass@80 matches the pass@256 of GRPO on the same model, demonstrating a 3x improvement in test-time sample efficiency. 
\end{enumerate}
Overall, our findings suggest that deliberate exploration---with the right notion of diversity---is a practical path toward discovery of new behaviors beyond sharpening.\footnote{Website and code: \url{https://rep-exp.github.io}}\loose

%% file: sections/introduction.tex
Reinforcement learning (RL) promises to endow agents with the ability to discover valuable
behaviors autonomously, via closed-loop trial and error. For language modeling
tasks with verifiable rewards, such as mathematical reasoning and code generation, post-training with reinforcement learning
has already enabled impressive breakthroughs
\citep{deepseek2025r1,openai2024o1}. Still, it is unclear whether contemporary RL implementations for language models attain the full promise of reinforcement learning. 
  Rather than unlocking capabilities not present in the pre-trained model, there is increasing evidence \citep{yue2025does,gandhi2025cognitive}
that existing RL recipes
\citep{schulman2017proximal,rafailov2024direct,deepseek2025r1} may
simply amplify or \emph{sharpen} \citep{huang2025self} behaviors that the base model can already execute, albeit with 
 modest probability. While this can
 be mitigated through deliberate data curation and some algorithmic interventions \citep{he2025rewarding,liu2025prorl,setlur2025e3}, data scale and quality are rapidly becoming bottlenecks, particularly in complex, open-ended domains where existing interventions fall short of eliciting desired behavior.

We argue that deliberate exploration---incentivizing the model to discover truly novel and diverse behavior---is an essential ingredient in realizing the full potential of RL for language model reasoning. 
Exploration has a rich history in both the theory and practice of RL, and exploration techniques tailored to deep networks
\citep{tang2017exploration,pathak2017curiosity,burda2018exploration,osband2019deep} have received extensive investigation in the context
of embodied decision making, including game playing and robotic control. %
These algorithms
proceed from scratch, without pre-training, yet rapidly learn complex behaviors, demonstrating that they enable learning beyond the sharpening regime.
If we can equip language models with exploration in a similar fashion, we may be able to advance reasoning capabilities without incurring exorbitant data curation costs. \loose %

In spite of the potential benefits of exploration, it is unclear which, if any, exploration technique from
deep RL can be scaled to modern language models. %
A central challenge involves the scalable quantification of novelty and behavior diversity---and acting 
on this information---when the decision space under consideration is the
combinatorially large space of language. At the same time, pre-trained language models contain tremendous prior knowledge compared to policies found in traditional embodied settings, which may be the key to guiding efficient exploration. This leads us to ask: \loose

\begin{minbox}{}
\begin{enumerate}
  \item Can the knowledge in pre-trained representations guide the search for novel behaviors?\loose \vspace{0.5em}
  \item Does deliberate exploration have the potential to move beyond sharpening the base model?\loose
  \end{enumerate}
\end{minbox}

\begin{figure}[t]
  \centering
  \includegraphics[width=\linewidth]{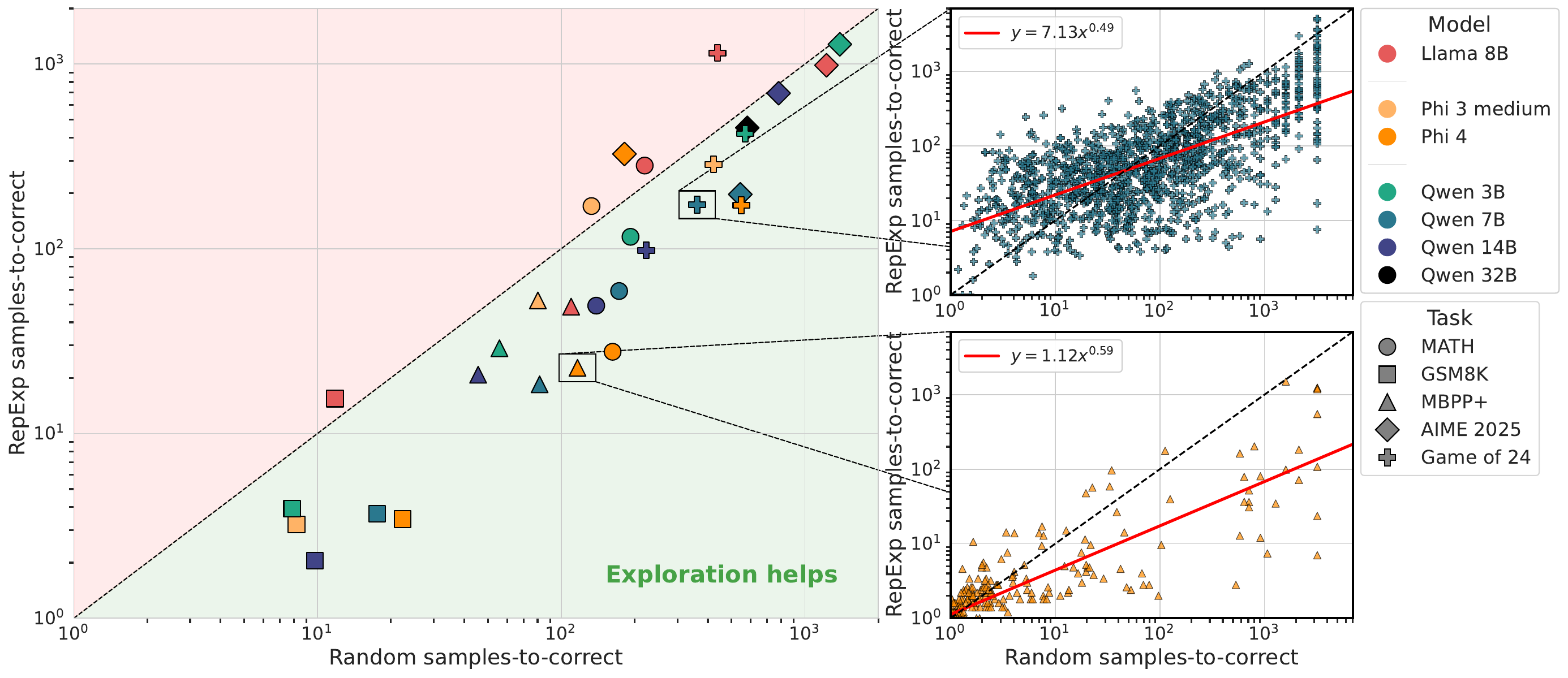}
  \iclr{\vspace{-0.6cm}}
  \arxiv{\vspace{-0.5cm}}
  \caption{\textbf{%
Representation-based inference-time exploration improves verifier efficiency.} \textbf{\emph{(Left)}} We plot the \emph{samples-to-correct}, the average number of samples until a correct response is selected, for a wide range of tasks and models. We compare two inference-time exploration methods: representation-based exploration (\cref{sec:methods}) and naive (random) sampling from the base model. \textbf{\emph{(Right)}} We display samples-to-correct, disaggregated to each question in the dataset, for two model-task pairs. \textbfc{We find representation-based exploration improves over random sampling for most model-task pairs.} For example, for \texttt{Qwen-2.5-14b-Instruct} we obtain over 50\% improvement in verifier efficiency on \texttt{GSM8K}, \texttt{MATH}, \texttt{MBPP+}, and \texttt{Game-of-24}. See~\cref{sec:coresets} for details.
\iclr{\vspace{-0.4cm}}
\arxiv{\vspace{-0.1cm}}
}%

  \label{fig:overview}
\end{figure}

\subsection{Contributions}
Toward answering these questions, we focus on understanding whether exploration with \emph{diversity bonuses} $\div(x,y)$ derived from a language model can effectively guide the search for diverse behaviors. We adopt a novel methodology (\cref{sec:preliminaries}) in which we first evaluate exploration in a simple, purely inference-time setting, then integrate our findings into post-training.\loose

\textbf{The inference-time selection problem (\cref{sec:preliminaries}).} In this setting, we aim to select a small set of responses $y_1,\ldots,y_{k}$ from a large set of candidates $y_1,\ldots,y_{N}$ for a given prompt $x$, such that the chosen set is as diverse as possible, and has high probability of including a positive response. This simple regime allows us to disentangle the role of diversity $\div(x,y)$ from other complex RL mechanisms, such as optimization and generalization.\loose

\textbf{Representation-based exploration improves diversity and efficiency.} Our main finding is that exploration with a \textbfc{representation-based} bonus (\cref{sec:methods}) derived from the pre-trained language model's hidden states significantly improves diversity and pass@k rates---both for our inference-time setting and for post-training. Our specific findings are as follows:
\loose
\begin{enumerate}
    \item \textbf{Inference-time (\cref{sec:inference_time_exploration}).}
    Inference-time exploration with representation-based diversity improves verifier efficiency. 
    For example, we obtain over 50\% improvement in verifier efficiency over standard sampling for \texttt{Qwen-2.5-14b-Instruct} on \texttt{GSM8K}, \texttt{MATH}, \texttt{MBPP+} and \texttt{Game-of-24}. 
    See~\cref{fig:overview} for an overview of our results for this setting. 
    \item \textbf{Post-training (\cref{sec:rl}).}
    Representation-based exploration can be incorporated into RL post-training,
    where its pass@$k$ performance is competitive with both GRPO and the base model \emph{uniformly for all $k$} (\cref{fig:rl_results}). 
    Notably, \textbfc{representation-based exploration completely eliminates the ``diversity collapse'' phenomenon where RL degrades pass@$k$ with respect to the base model for large $k$}~\citep{dang2025weight,yue2025does,wu2025invisible}. In addition, representation-based exploration induces responses that look much more novel under the base model (\cref{fig:log_prob_histogram_math}). 
\end{enumerate}
Our findings, particularly these last two points, suggest that deliberate exploration
is a practical path toward discovery of new behaviors beyond sharpening.
Although our experiments focus on arguably the simplest principled representation-based exploration scheme---for which we already see substantial performance improvements---we expect that our two-pronged evaluation approach will enable a deeper understanding of the benefits and tradeoffs of more sophisticated strategies, which may help realize the full potential of reinforcement learning for language model reasoning.

\paragraph{Diversity-guided generation (\cref{sec:guided})}
As a proof of concept, we also evaluate an inference-time exploration algorithm that uses representation-based diversity to encourage exploration \emph{during the
autoregressive generation process itself}. We find that this improves pass@k for large k over naive sampling for \texttt{Qwen-2.5-7b-Instruct} on \texttt{MATH}.

  \begin{figure}[t]
    \centering
        \includegraphics[width=\linewidth]{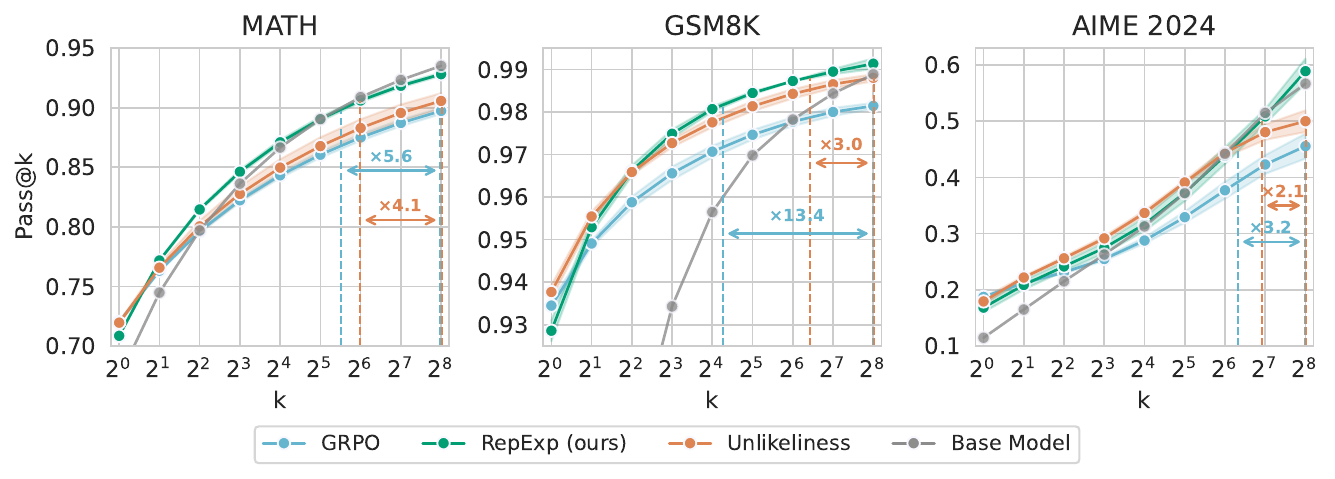}
    \vspace{-0.5cm}
    \caption{\textbf{Pass@k for RL post-training with exploration.} We find that RL generally increases the pass@k for small values of $k$ compared to the base model, but that exploration is required to improve or even preserve base model pass rates for larger values of $k$. For \texttt{MATH} and \texttt{GSM8K}, \repexp roughly matches or improves upon Unlikeliness for $k \ge 2$. For \texttt{AIME 2024}, \repexp is slightly worse than Unlikeliness until $k = 64$, after which it surpasses Unlikeliness for all larger values of $k$. 
    Shaded areas indicate one standard error.
    Horizontal arrows indicate the test-time sample efficiency improvement for pass@256 of \repexp over GRPO (blue) and Unlikeliness (orange). 
    \textbfc{\textsf{RepExp} is 2.1-4.1x more sample-efficient than Unlikeliness and 3.2-13.4x more sample-efficient than GRPO.} 
    \iclr{\vspace{-0.4cm}}
    }
    \label{fig:rl_results}
\end{figure}

\section{Problem Setup: From Inference-Time to Post-Training}
\label{sec:preliminaries}

In this section, we describe the two problem settings we consider for exploration: inference-time selection and RL post-training. In what follows, $\pi$ denotes a language model that maps a prompt $x\in\cX$ to a distribution over responses $y \in \cY$, and $\rstar(x,y)\in\crl{0,1}$ denotes a verifiable reward function that measures correctness at a task of interest, such as whether the answer to a math question is correct, or whether a
Python program passes unit tests. 

\paragraph{Methodology and motivation}
The goal of RL post-training is to find a policy $\pi$ that maximizes the expected reward $\En_{y \sim \pi(\cdot\mid{}x)}\brk*{r^\star(x,y)}$. 
Given a budget $k$ of verifier queries per question at each data collection round, post-training algorithms such as GRPO \citep{shao2024deepseekmath} update the model iteratively, where in each iteration they sample $k$ responses $y_1,\ldots,y_{k}\iidsim\pi(\cdot\mid{}x)$ per prompt $x$ from the current model $\pi$, query the verifier for a reward $\rstar(x,y_i)$ for each response, and use observed rewards to update the model for the next iteration.\loose

If the initial model $\pi$ has poor support over rewarding behavior---i.e., if $r^\star(x,y)=0$, with high probability under $y \sim \pi(\cdot\mid{}x)$---common RL algorithms such GRPO or PPO~\citep{schulman2017proximal} will not make any progress. This motivates interventions for exploration such as bonuses \citep{tang2017exploration,pathak2017curiosity,burda2018exploration,osband2019deep} and alternative sampling strategies~\citep{holtzmancurious,minhturning}. However, understanding the benefits and tradeoffs of these interventions in RL post-training is challenging because exploration interacts with optimization and generalization. To isolate exploration from these other
considerations, we center our investigation around a task we refer to as \textbfc{inference-time selection}, validating interventions in this setting before integrating them into the RL post-training pipeline.
\loose

\subsection{Inference-time selection}
\label{sec:inference-time-selection}

In the inference-time selection problem, we aim to use a fixed model $\pi$ to build a set of $k$ responses to a given prompt $x$ that are maximally diverse and have high probability of containing a positive response. As a simple baseline, we may independently sample $k$ responses from the model---potentially with high-temperature sampling, nucleus or min-p sampling, or other modified sampling schemes. However, the limitations of these baselines are (i) they may not effectively capture the model's understanding of diversity, and (ii) by sampling independently, we may waste verifier queries on redundant responses. \loose 

Instead, we focus on \emph{selection-based} approaches that initially sample a large set of candidate responses to the prompt, then use a diversity bonus $\div(x,y)$ derived from the model to filter this set down to a smaller, more diverse ``coreset'' \citep{clarkson2010coresets,feldman2020turning} that is passed to the verifier. Formally, we consider the following protocol: For each prompt $x$, we (1) sample an initial batch of $N$ responses $y_1,\ldots,y_{N}\sim\pi(\cdot\mid{}x)$, (2) use the inference-time selection algorithm $\Alg$ to  select $k$ of these responses (a subset $S\subset[N]$ of size $|S|=k$), and (3) query the verifier and record if any of the selected responses are rewarding. That is, we measure pass@$k$,\loose 
{%
\setlength{\abovedisplayskip}{0.5em}
\setlength{\belowdisplayskip}{0.3em}
\begin{align}
  \label{eq:passatk}
    \En_{y_1,\ldots,y_{N}\sim\pi(\cdot\mid{}x)}\brk[\big]{\En_{S \sim \Alg(x,y_1,\ldots,y_{N})}\brk[\big]{\max_{i \in S}[ r^\star(x,y_i)]}}.
\end{align}
}%
Importantly, the filtering algorithm operates without the verifier, and so successfully retaining high-quality responses translates to improved \emph{verifier efficiency} (i.e., number of responses for which we query the verifier) over the initial set of responses. 
Thus, a useful diversity bonus $\div(x,y)$ should yield a coreset that is maximally ``exploratory,'' in the sense that it is the most diverse set of responses that can be selected for a fixed budget of verifier queries. For example, in math reasoning settings, we would like to select the distinct-but-plausible proof strategies for a given problem, thus covering the space of potential proofs and maximizing the chance of selecting a correct one.

\begin{remark}\label{remark}
  While we mainly introduce inference-time selection as a stepping stone to post-training (i.e., algorithms in this setting are not more compute-efficient than naive sampling, even if they are more verifier-efficient), we do expect inference-time exploration to be useful in its own right for domains where querying a verifier is costly or difficult
  (e.g., collecting feedback from expert-level annotators), allowing for more sample- and hence cost-efficient data curation. For preliminary results in one such domain, please refer to~\arxiv{\cref{subsec:application}}\iclr{\cref{sec:appendix_protein_sequence_generation}}.
\end{remark}

\subsection{Reinforcement learning post-training}
As described earlier, RL post-training (e.g., with GRPO or PPO) proceeds by iteratively sampling batches of responses, querying the verifier, and using the feedback to update the current policy. After selecting a checkpoint $\hat{\pi}$, we evaluate performance via pass@$k$ under standard generation, $\En_{y_1,\ldots,y_k \sim \hat{\pi}(\cdot\mid{}x)}\brk[\big]{\max_{i \in [k]} r^\star(x,y_i)}$.
There are two natural approaches to integrate exploration methods into this process. The first is to adjust the independent sampling process only (e.g., through nucleus sampling or min-p sampling), and the second is to augment the training objective with an exploration bonus $\div(x,y)$. Our experiments focus on the latter approach; however, based on our results for inference-time selection, we expect that incorporating representation-based exploration into the sampling process will also improve RL post-training performance. Indeed, our two-pronged evaluation is motivated by the hypothesis that \emph{diversity bonuses $\div(x,y)$ that perform well at inference-time also perform well in post-training.}\loose

\begin{figure}
  \input{sections/repexp_figure}
  \iclr{\vspace{-0.3cm}}
  \caption{\textbf{\textsf{RepExp} for inference-time exploration.} Given a prompt, \textsf{RepExp} selects a diverse set of responses from a large pool by optimizing elliptical bonuses computed using representations from the language model.}
  \iclr{\vspace{-0.4cm}}
  \label{fig:repexp_figure}
\end{figure}

%% file: sections/repexp_figure.tex
\centering
\usetikzlibrary{positioning,arrows.meta,shapes,calc,backgrounds}
\begin{tikzpicture}[
    scale=0.9,
    font=\sffamily\footnotesize,
    stepbox/.style={
        rectangle,
        draw=violet!60,
        line width=1pt,
        fill=violet!8,
        rounded corners=3pt,
        minimum width=3.5cm,
        minimum height=0.6cm,
        align=center,
        font=\sffamily\footnotesize\bfseries,
        text=violet!80!black
    },
    prompt/.style={
        rectangle,
        draw=green!60!black,
        line width=1pt,
        fill=green!5,
        rounded corners=3pt,
        minimum width=2.5cm,
        minimum height=0.5cm,
        align=center,
        font=\sffamily\footnotesize\bfseries
    },
    response/.style={
        rectangle,
        draw=blue!50,
        line width=1pt,
        fill=blue!5,
        rounded corners=2pt,
        minimum width=3cm,
        minimum height=0.45cm,
        align=center,
        font=\sffamily\scriptsize
    },
    embed/.style={
        rectangle,
        draw=orange!60,
        line width=1pt,
        fill=orange!10,
        rounded corners=2pt,
        minimum width=1cm,
        minimum height=0.4cm,
        align=center,
        font=\sffamily\scriptsize\bfseries
    },
    point/.style={
        circle,
        fill=gray!40,
        inner sep=1.5pt,
        draw=gray!60
    },
    selected/.style={
        circle,
        fill=red!80,
        inner sep=2pt,
        draw=red!90,
        line width=0.8pt
    },
    arrow/.style={
        ->,
        >=stealth,
        thick,
        gray!70
    },
    flowarrow/.style={
        ->,
        >=stealth,
        very thick,
        blue!60
    }
]

\begin{scope}[local bounding box=col1]
    \node[stepbox] (step1) at (-6, 3) {{Step 1: Generate}};
    \node[prompt, below=0.4cm of step1, text=green!30!black] (prompt) {Prompt $x$};
    
    \node[response, below=0.6cm of prompt] (r1) {$y_1$: ``Let $x = \sqrt{2}$...''};
    \node[response, below=0.2cm of r1] (r2) {$y_2$: ``Set $a^2 + b^2$...''};
    \node[response, below=0.2cm of r2] (r3) {$y_3$: ``Note that...''};
    \node[below=0.15cm of r3] (dots1) {$\vdots$};
    \node[response, below=0.15cm of dots1] (rN) {$y_N$: ``First, consider...''};
\end{scope}

\begin{scope}[local bounding box=col2]
    \node[stepbox] (step2) at (-2, 3) {{Step 2: Embed}};
    
    \coordinate (embedstart) at (-2, 1.5);
    \node[embed] (e1) at (-2, 0.85) {$\bar{h}_1$};
    \node[embed, below=0.3cm of e1] (e2) {$\bar{h}_2$};
    \node[embed, below=0.3cm of e2] (e3) {$\bar{h}_3$};
    \node[below=0.15cm of e3] (dots2) {$\vdots$};
    \node[embed, below=0.2cm of dots2] (eN) {$\bar{h}_N$};
    
    \node[font=\sffamily\scriptsize\itshape, above=0.1cm of e1] {$\bar{h}_\theta(x,y_i)$};
\end{scope}

\begin{scope}[local bounding box=col3]
    \node[stepbox] (step3) at (2, 3) {{Step 3: Select}};
    
    \begin{scope}[shift={(2, 0)}]
        \draw[->, thin, gray!50] (-1.8,0) -- (1.8,0);
        \draw[->, thin, gray!50] (0,-1.8) -- (0,1.8);
        
        \begin{scope}[on background layer]
            \draw[blue!30, line width=1pt, fill=blue!5] (0,0) circle (1.5);
            \draw[orange!40, line width=1pt, rotate=30, fill=orange!5] (0,0) ellipse (1.3cm and 0.8cm);
            \draw[purple!50, line width=1pt, rotate=60, fill=purple!5] (0,0) ellipse (1.0cm and 0.5cm);
        \end{scope}

        \coordinate (p1) at (0.85, 0.95);
        \coordinate (p2) at (-0.8, 0.75);
        \coordinate (p3) at (0.45, -1.2);
        \coordinate (p4) at (1.3, -0.35);
        \coordinate (p5) at (-1.1, -0.6);
        \coordinate (p6) at (-0.85, 0.35);
        \coordinate (p7) at (0.4, -0.45);
        \coordinate (p8) at (-0.3, 1.15);
        
        \foreach \i in {1,...,8} {
            \node[point] at (p\i) {};
        }
        
        \node[selected] (s2) at (p2) {};
        \node[selected] (s5) at (p5) {};
        \node[selected] (s7) at (p7) {};
        \node[font=\sffamily\scriptsize\bfseries, red!90, above right=-0.5em of s2] {$y_2$};
        \node[font=\sffamily\scriptsize\bfseries, red!90, above right=-0.5em of s5] {$y_5$};
        \node[font=\sffamily\scriptsize\bfseries, red!90, above right=-0.5em of s7] {$y_7$};
        
        \node[font=\sffamily\footnotesize] at (1.25, 1.25) {\textcolor{blue!60}{$\Lambda_0$}};
        \node[font=\sffamily\footnotesize] at (-1.25, 0.25) {\textcolor{orange!70}{$\Lambda_1$}};
        \node[font=\sffamily\footnotesize] at (0.9, 0.5) {\textcolor{purple!70}{$\Lambda_2$}};
    \end{scope}
\end{scope}

\arxiv{
\begin{scope}[local bounding box=col4]
    \node[stepbox] (step4) at (6, 3) {{Step 4: Output}};
    
    \node[rectangle, draw=gray!70, line width=1pt, fill=gray!5, rounded corners=3pt,
          minimum width=3.5cm, align=left, font=\sffamily\scriptsize, inner sep=4pt, text=black!80!gray] 
          (math) at (6, -2.2) {
        \textbf{Selection:}\\[2pt]
        $y_{t+1} = \arg\max_{y} \bar{h}(x,y)^\top \Lambda_t \bar{h}(x,y)$.\\[4pt]
        \textbf{Update:}\\[2pt]
        $\Lambda_t = \Lambda_{t-1} - \frac{\Lambda_{t-1} \bar{h}_t\bar{h}_t^\top \Lambda_{t-1}}{1+\bar{h}_t^\top \Lambda_{t-1} \bar{h}_t}$.
    };
    
    \node[rectangle, draw=green!60!black, line width=1pt, fill=green!5, rounded corners=3pt,
            minimum width=3cm, font=\sffamily\footnotesize, text=green!30!black] (output) at (6, 0.5) {
        \textbf{Selected:} $\{y_2, y_5, y_7\}$
    };
\end{scope}
}

\iclr{
\begin{scope}[local bounding box=col4]
    \node[stepbox] (step4) at (6, 3) {{Step 4: Output}};
    
    \node[rectangle, draw=gray!70, line width=1pt, fill=gray!5, rounded corners=3pt,
          minimum width=3.5cm, align=left, font=\sffamily\scriptsize, inner sep=4pt, text=black!80!gray] 
          (math) at (6, -2.2) {
        \textbf{Selection:}\\[2pt]
        $y_{t+1} = \arg\max_{y} \bar{h}(x,y)^\top \Lambda_t \bar{h}(x,y)$.\\[4pt]
        \textbf{Update:}\\[2pt]
        $\Lambda_t = \Lambda_{t-1} - \frac{\Lambda_{t-1} \bar{h}_t\bar{h}_t^\top \Lambda_{t-1}}{1+\bar{h}_t^\top \Lambda_{t-1} \bar{h}_t}$.
    };
    
    \node[rectangle, draw=green!60!black, line width=1pt, fill=green!5, rounded corners=3pt,
            minimum width=3cm, font=\sffamily\footnotesize, text=green!30!black] (output) at (6, 0.5) {
        \textbf{Selected:} $\{y_2, y_5, y_7\}$
    };
\end{scope}
}

\draw[flowarrow] (r1.east) -- (e1.west);
\draw[flowarrow] (r2.east) -- (e2.west);
\draw[flowarrow] (r3.east) -- (e3.west);
\draw[flowarrow] (rN.east) -- (eN.west);

\draw[line width=1.2pt, color=blue!60] (e1.east) -- (.5, 0.0);
\draw[line width=1.2pt, color=blue!60] (e2.east) -- (.5, 0.0);
\draw[line width=1.2pt, color=blue!60] (e3.east) -- (.5, 0.0);
\draw[line width=1.2pt, color=blue!60] (eN.east) -- (.5, 0.0);

\node[rectangle, draw=red!60, line width=1pt, fill=red!5, rounded corners=3pt, text=red!60!black,
    minimum width=10cm, minimum height=0.7cm, align=center] at (0, -4.05) {
    {\normalsize\textbf{\texttt{RepExp:}} \textsf{Iteratively select responses that maximize elliptical bonus $\bar{h}^\top \Lambda_t \bar{h}$ for diversity.}}
};

\end{tikzpicture}
\label{fig:repexp}

%% file: sections/main.tex
Having motivated our setup, we now turn to the question of what diversity bonuses $\div(x,y)$ are suitable for exploration with language models. While many metrics have been proposed in the literature \citep{tang2017exploration,pathak2017curiosity,burda2018exploration,osband2019deep}, the challenge in adapting these techniques to language models is to simultaneously (i) capture the model's understanding and (ii) allow for efficient computation at scale. For example, count-based exploration~\citep{tang2017exploration} is simple, but unsuited to large decision spaces. On the other hand, approaches based on intrinsic curiosity \citep{pathak2017curiosity}, random network distillation \citep{burda2018exploration}, and posterior sampling~\citep{osband2019deep} are better suited to large or continuous spaces, but require additional learning machinery (i.e., auxiliary networks), which introduces significant complexity when scaling to language models.

\label{sec:representation}

We focus our experiments on an exploration strategy that avoids these shortcomings: An adaptation of elliptic bonuses and sampling---a de facto standard for linear bandits and active learning~\citep{abbasi2011improved,chu2011contextual,ash2021anti,henaff2022exploration, saran2023streaming, foster2025good}\footnote{Indeed, elliptical bonuses are ubiquitous in linear bandits and reinforcement learning, the simplest non-tabular RL setting, where they have strong provable guarantees. Beyond this, elliptic bonuses and iterative schemes such as \cref{alg:elliptic_coreset} have a long history in the theory of optimal experimental design \citep{kiefer1960equivalence,pukelsheim2006optimal,allen2021near} and active learning~\citep{cesa2009robust,agarwal2013selective,gu2014batch,chaudhuri2015convergence}}---with a representation derived from the language model's hidden states. This approach is arguably the simplest principled strategy that is appropriate for language models,  and already yields significant performance improvements in our experiments. Our use of elliptic bonuses is particularly inspired by \citet{foster2025good}, who prove that test-time exploration with such bonuses has provable computational benefits in a simplified language model setting with frozen features.\loose

At a high level, elliptical bonus methods operate over a $d$-dimensional feature space and adopt a linear-algebraic measure of novelty: given previously seen feature vectors $h_1,\ldots,h_{i-1}$ the novelty (or bonus) of a new feature vector $h$ is defined as
{%
\setlength{\belowdisplayskip}{-0.1cm}}
\begin{align}
    \label{eq:elliptic_bonus}
    \div(h\mid{}h_{1:i-1}) = h^\top \Sigma_i^{-1} h 
    \qquad 
    \Sigma_i = \lambda I_d + \sum_{j < i} h_j h_j^\top
\end{align}
\iclr{\vspace{-0.5cm}}

These bonuses are grounded in the theory of linear regression: If we fit a linear model $f_\theta(h) = \inner{\theta}{h}$ on features $h_1,\ldots,h_{i-1}$ (with associated regression targets), the prediction error on $h$ will be bounded by $\div(h\mid{}h_{1:i-1})$ \citep{lattimore2020bandit}. Thus, $\div(h\mid{}h_{1:i-1})$ reflects novelty, as it will be large for features $h$ that are poorly represented by the training dataset.\loose

To adapt elliptical bonuses to language models, we use representations extracted from the model itself as the feature vectors. Formally, given a prompt $x$ and a response $y_i=y_i^1,\ldots,y_i^T$ of $T$ tokens, we form the feature vector as $\bar{h}_\theta(x,y_i) := \frac{1}{T}\sum_{t=1}^T h_\theta(x,y_i^{1:t})$ where $h_\theta(x,y_i^{1:t})\in \mathbb{R}^d$ is the last-layer hidden state of the model on input $(x,y_i^{1:t})$ (the activation prior to the unembedding matrix). In~\cref{fig:representation_ablation}, we ablate this choice by comparing it to the effectiveness of using representations at the last token $h_\theta(x, y_i^{1:T})$ or penultimate token $h_\theta(x, y_i^{1:{T-1}})$ instead. We reduce dimensionality to 512 using a sparse random projection~\citep{Li2006VerySR}; see \cref{sec:appendix_inference_time} for details. \loose

\vspace{-0.05cm}
\paragraph{Representation-based exploration for inference-time selection}
\cref{fig:repexp_figure} presents \repexp, our main algorithm for inference-time selection using representation-based elliptical bonuses. Here, given a single prompt $x$ and a set of candidate generations $\cY=\crl*{y_1, \ldots, y_N}$, we iteratively select the generation that maximizes the elliptical bonus via $y_{t + 1} = \argmax_{y\in\cY}\bar{h}_\theta(x,y)\Sigma_t^{-1}\bar{h}_\theta(x,y)$, leveraging the representations $\bar{h}_\theta(x,y)$ described above. We efficiently update the inverse covariance matrix $\Sigma^{-1}$ using the Woodbury identity for $\bigoh(d^2)$ time per step~\citep{vetterling1992numerical}. We formally present our procedure in~\cref{alg:elliptic_coreset}. \loose

\arxiv{
\begin{figure}[t]
    \centering
    \adjustbox{valign=t}{\begin{minipage}{0.40\textwidth}
        \input{sections/algorithm_elliptic} 
    \end{minipage}}\hfill%
    \adjustbox{valign=t}{\begin{minipage}{0.57\textwidth}
        \centering
        \includegraphics[width=.75\linewidth]{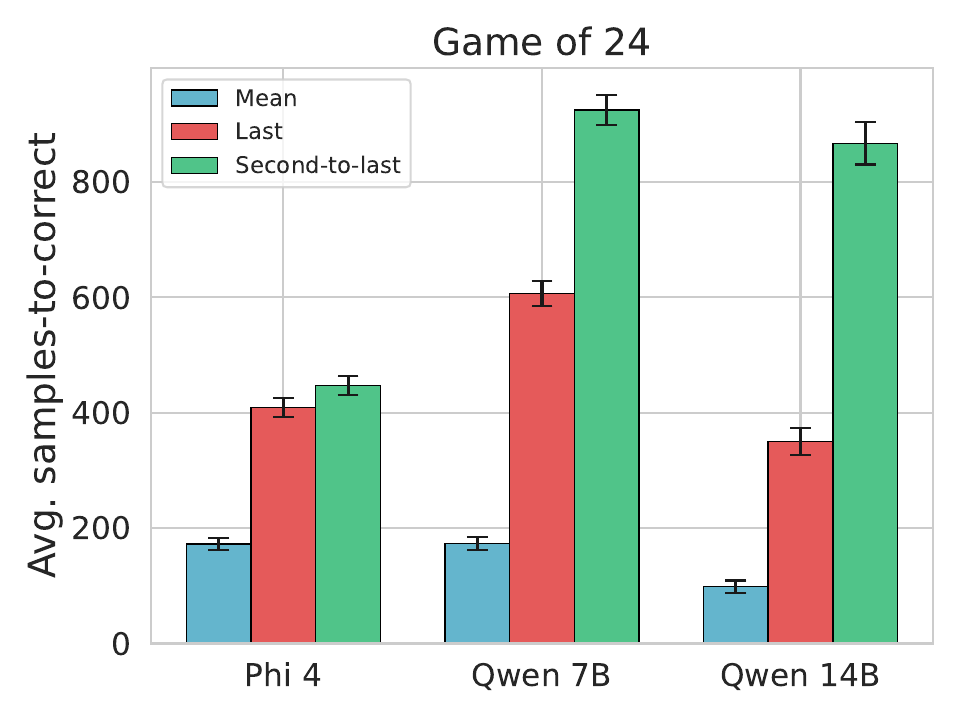}
        \vspace{-0.4cm}
        \caption{\textbf{Representation ablation.} We compare averaging all token representations to using those from the penultimate or final token. \textbfc{Averaging is over 2x more sample efficient.}}
        \label{fig:representation_ablation}
    \end{minipage}}
    \end{figure}
}
\iclr{
\begin{figure}[t]
    \centering
    \adjustbox{valign=t}{\begin{minipage}{0.46\textwidth}
        \input{sections/algorithm_elliptic} 
    \end{minipage}}\hfill%
    \adjustbox{valign=t}{\begin{minipage}{0.52\textwidth}
        \centering
        \includegraphics[width=.87\linewidth]{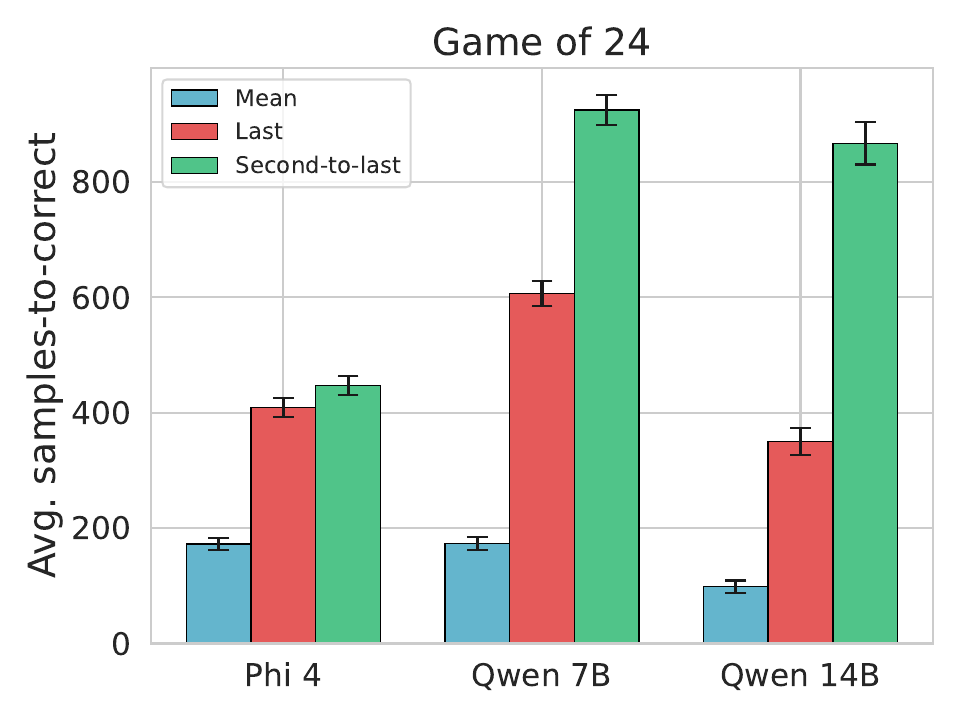}
        \vspace{-0.4cm}
        \caption{\textbf{Representation ablation.} We compare averaging all token representations to using those from the penultimate or final token. \textbfc{Averaging is over 2x more sample efficient.}}
        \label{fig:representation_ablation}
    \end{minipage}}
    \iclr{\vspace{-0.3cm}}
    \end{figure}
}

\vspace{-0.05cm}
\paragraph{Representation-based exploration for RL post-training}
For our post-training experiments, we use the same representations $\bar{h}_\theta(x,y)$ as above, but directly augment the rewards with elliptic bonuses instead of performing coreset selection. Concretely, given the current iterate $\pi_\theta$ in GRPO, we first sample a group of $k$ responses $y_1,\ldots,y_k\iidsim\pi_\theta(\cdot\mid{}x)$ for each prompt $x$. Letting $\Sigma\ldef\lambda{}I_d + \sum_{i=1}^{k}\bar{h}_\theta(x,y_i)\bar{h}_\theta(x,y_i)^{\trn}$, we define the reward for response $y_i$ as\footnote{The bonus here can be interpreted as a \emph{leverage score} for $y_i$ \citep{drineas2006subspace,cohen2015uniform}.} $\rstar(x,y_i) + \beta\cdot \bar{h}_\theta(x,y_i)^\top \Sigma^{-1} \bar{h}_\theta(x,y_i)$,
where $\beta>0$ is a bonus parameter. While one could also imagine performing inference-time coreset selection in the loop with GRPO, this approach is more practical and efficient, and it achieves significant improvements in performance. We refer the reader to~\cref{sec:rl} and~\cref{sec:appendix_rl} for further details.

\oldparagraph{Why representation-based elliptical bonuses?} We summarize several desirable properties of these bonuses. First, by leveraging the hidden state of the model in featurization, the bonuses capture rich information about the generations, thereby incorporating the language model's prior knowledge. Second, the method is \emph{history-aware}\footnote{In the RL setting, we do not let the covariance matrix persist across multiple iterations of the same question, and hence there it is only \emph{group-aware}.}: the covariance matrix summarizes all previously selected generations, and redundancy with previous selection (in representation space) is penalized. Finally, the method is simple and scalable, involving no additional learning machinery and using rank-one updates to avoid costly matrix inversions.\loose

\section{Inference-Time Exploration: Experimental Results}
\label{sec:inference_time_exploration}

In this section, we investigate the performance of representation-based exploration for the inference-time selection problem. 
We detail the experimental setup in \cref{sec:setup}, present main findings in \cref{sec:coresets}, and present additional experiments with a ``token-level'' variant in \cref{sec:guided}.%

\label{sec:setup}

\iclr{
\begin{figure}[t]
    \centering
    \includegraphics[width=\linewidth]{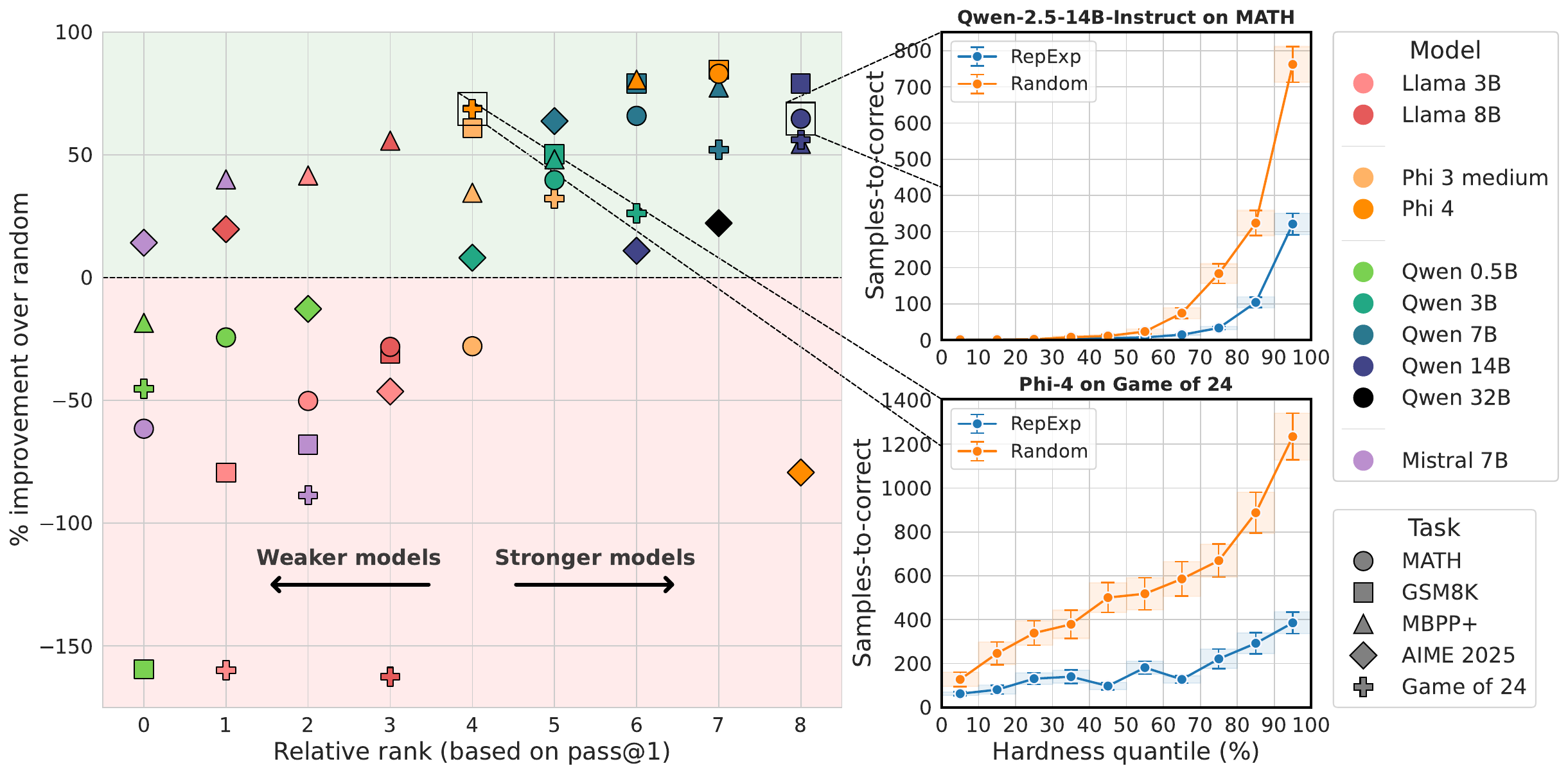}
    \iclr{\vspace{-0.3cm}}
    \arxiv{\vspace{-0.3cm}}
    \caption{\textbf{A closer look into when \repexp provides improvement.} \textbf{\emph{(Left)}} For each task, we rank models according to their pass@1 rate (the weakest model has rank 0, and the strongest has rank 8). We then plot relative improvement (\%) of \repexp over random sampling, sorting by rank on the x-axis. While \repexp can hurt weaker models (e.g., \texttt{Qwen-2.5-0.5B-Instruct}), we find \textbfc{stronger models almost always benefit from exploration} (e.g., \texttt{Qwen-2.5-14B-Instruct}). \textbf{\emph{(Right)}} For two different model-task pairs, we plot the samples-to-correct as a function of \emph{question hardness}. Hardness is measured by the samples-to-correct from a high-quality third-party model (\texttt{GPT-4o mini}). We find that \textbfc{\textsf{RepExp} has the greatest benefit on harder examples} (e.g., the hardest 20\% of questions on \texttt{MATH}). Shaded areas indicate one standard error.\loose
    \iclr{\vspace{-0.4cm}}
    }
    \label{fig:model_quality_vs_improvement}
  \end{figure}
}

\paragraph{Datasets} We use the \emph{test} splits of the following five datasets: \texttt{MATH}~\citep{hendrycks2measuring}, \texttt{GSM8K}~\citep{cobbe2021training}, \texttt{MBPP+}~\citep{liu2023your}, \texttt{Game-of-24}~\citep{yao2023tree}, and \texttt{AIME 2025}. We chose these tasks as they cover easy (\texttt{GSM8K}), medium (\texttt{MATH}), and harder (\texttt{Game-of-24}, \texttt{AIME}) difficulty levels in math. In addition, we include \texttt{MBPP+} to verify that our findings transfer to the coding domain. For a more detailed overview of these datasets, please refer to \cref{sec:appendix_coresets}.

\paragraph{Models} We consider a range of model families and sizes: $\texttt{Phi-3-Medium}$~\citep{phi3} and $\texttt{Phi-4}$, $\texttt{Llama-3.2-3B-Instruct}$ and $\texttt{Llama-3.1-8B-Instruct}$ \citep{dubey2024llama}, \citep{phi4}, $\texttt{Qwen-2.5-X-Instruct}$ \citep{qwen2p5} for $\texttt{X}\in\crl*{\texttt{0.5B}, \texttt{3B}, \texttt{7B}, \texttt{14B}, \texttt{32B}}$, and $\texttt{Mistral-7B}$ \citep{mistral7b}.\loose %

\arxiv{
\begin{figure}[t]
    \centering
    \includegraphics[width=\linewidth]{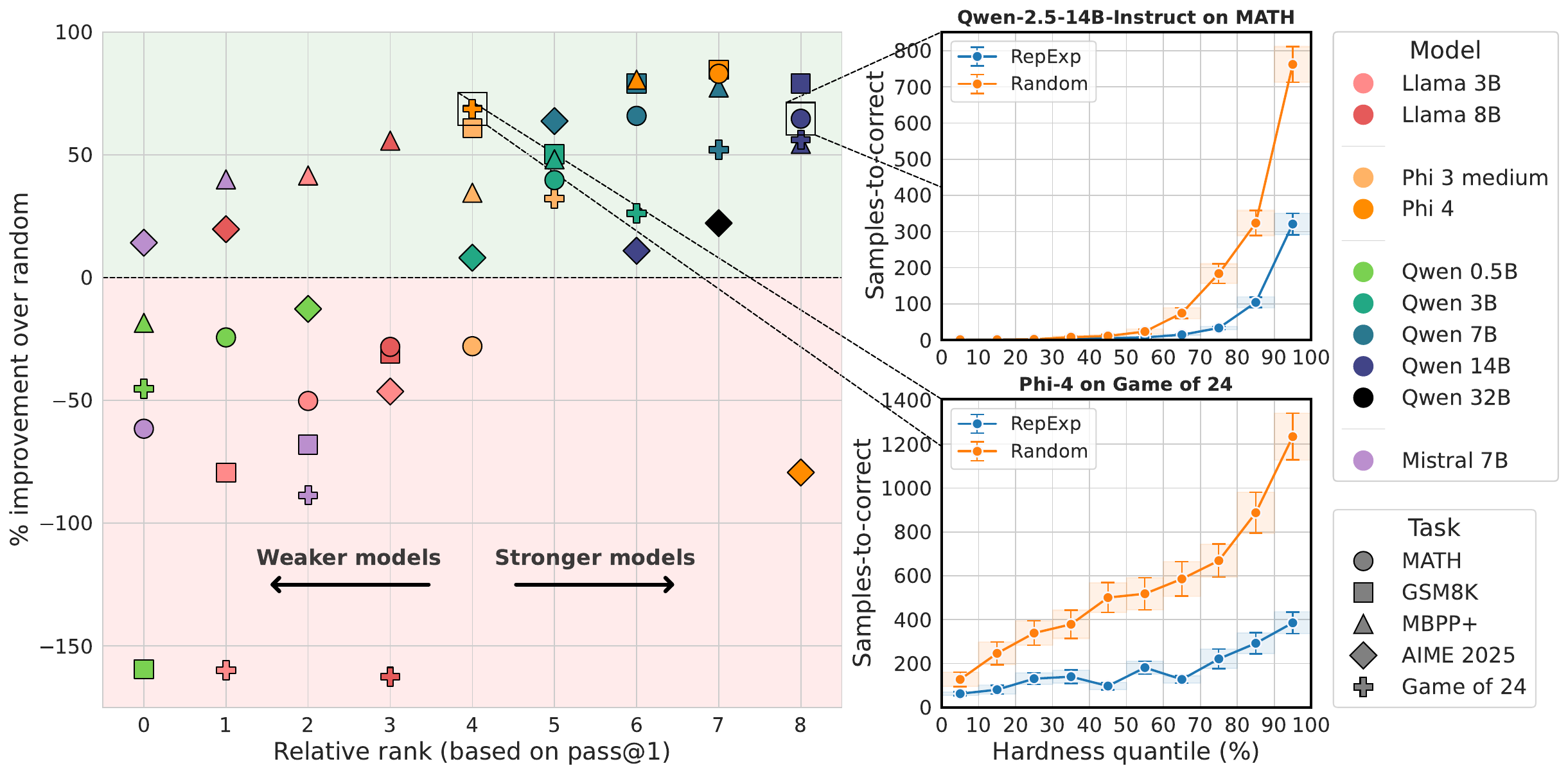}
    \iclr{\vspace{-0.3cm}}
    \arxiv{\vspace{-0.3cm}}
    \caption{\textbf{A closer look into when \repexp provides improvement.} \textbf{\emph{(Left)}} For each task, we rank models according to their pass@1 rate (the weakest model has rank 0, and the strongest has rank 8). We then plot relative improvement (\%) of \repexp over random sampling, sorting by rank on the x-axis. While \repexp can hurt weaker models (e.g., \texttt{Qwen-2.5-0.5B-Instruct}), we find \textbfc{stronger models almost always benefit from exploration} (e.g., \texttt{Qwen-2.5-14B-Instruct}). \textbf{\emph{(Right)}} For two different model-task pairs, we plot the samples-to-correct as a function of \emph{question hardness}. Hardness is measured by the samples-to-correct from a high-quality third-party model (\texttt{GPT-4o mini}). We find that \textbfc{\textsf{RepExp} has the greatest benefit on harder examples} (e.g., the hardest 20\% of questions on \texttt{MATH}). Shaded areas indicate one standard error.\loose
    \iclr{\vspace{-0.4cm}}
    }
    \label{fig:model_quality_vs_improvement}
  \end{figure}
}

\paragraph{Algorithms}
In our experiment protocol, we initially draw a pool of $N$ candidate generations from the base model, where unless otherwise specified we use temperature $\tau=1.0$ and $\text{top-p}=1.0$, which we refer to as vanilla settings (for \texttt{MBPP+}, we set $\text{top-p} = 0.95$).
Then we compare \repexp with budget $k$ with the baseline of random sampling (without replacement) of $k$ responses from this pool. 
We consider generating the pool using different samplers such as nucleus and min-p sampling in~\cref{fig:different_candidate_pools}, but always use random sampling without replacement as the baseline. 
See \cref{sec:appendix_coresets} for further details.\loose

\begin{figure}[t]
    \centering
    \includegraphics[width=\linewidth]{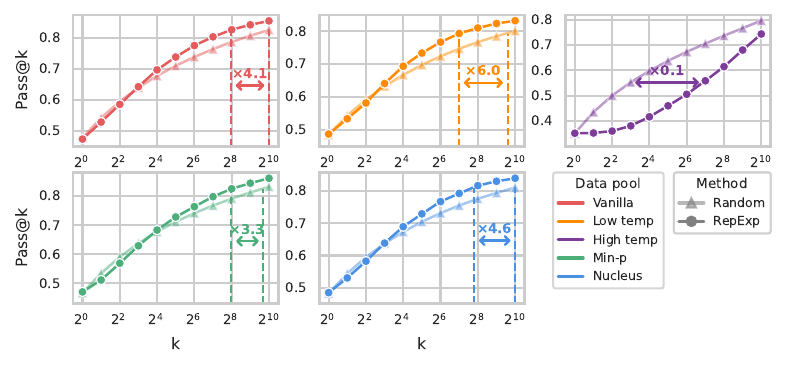}
    \vspace{-0.7cm}
    \caption{\textbf{Benefits of \repexp across data pools}, for the inference-time exploration setup in \cref{fig:overview} (\cref{sec:coresets}). We plot the pass@$k$ curve for random vs. \repexp across five different data pools (base samplers). \textbfc{\repexp on top of vanilla generation outperforms random sampling on top of \emph{any} of the generation strategies.} Moreover, except for the high-temperature pool, \repexp over a pool improves verifier efficiency, with \textbfc{3x to 6x improvement over random sampling} for that pool.
    \iclr{\vspace{-0.4cm}}
    }
    \label{fig:different_candidate_pools}
\end{figure}

\subsection{Benefits of representation-based exploration}
\label{sec:coresets}

We present our results as a series of Research Findings (RF), expanding on the findings in \cref{fig:overview}.

\paragraph{\underline{RF1}: \repexp improves verifier efficiency across models and tasks}
In \cref{fig:overview}, we plot the \emph{samples-to-correct}, defined as the expected number of samples $k$ with which we query the verifier before finding a correct answer, for all model-task pairs. We compare \repexp, which picks responses to a fixed question according to \cref{alg:elliptic_coreset}, with the random sampling baseline. For both algorithms, we average the samples-to-correct across all questions in the dataset. Our results show the bulk of the data fall below the line $y = x$, indicating exploration improves over random sampling in most cases. For example, we find \repexp obtains a 50\% improvement in samples-to-correct for \texttt{Qwen-2.5-14b-Instruct} in \texttt{MATH}, \texttt{GSM8K}, \texttt{MBPP+}, and \texttt{Game-of-24}. 

\paragraph{\underline{RF2}: The benefits of \repexp grow with model strength} Since \repexp relies on the model's internal representations, it is natural to hypothesize that weaker models might have worse representations and thus benefit less from exploration. To validate this hypothesis, we expand the collection of models in \cref{fig:overview} to include additional weaker models (e.g., \texttt{Qwen-2.5-0.5B-Instruct} and \texttt{Mistral-7b}). For each task, we rank models according to their pass@1 performance, and plot the relative improvement of representation-based exploration over random sampling in \cref{fig:model_quality_vs_improvement}. We indeed observe a strong correlation between model strength and the benefit from representation-based exploration: weaker models (e.g., \texttt{Qwen-2.5-0.5B}) experience no benefit or even degradation, while the strongest models (e.g., \texttt{Qwen-2.5-32B}) almost uniformly benefit.\loose

\paragraph{\underline{RF3}: \repexp provides more improvement for harder questions}
Beyond \textbf{RF2}---which provides insight into the benefits of \repexp across \emph{models}---we also evaluate the benefits across \emph{question difficulty}, for a fixed model and task. To this end, we sort all questions for a given task by their samples-to-correct under random sampling with a reference model (\texttt{GPT-4o-mini}). We then group the questions in bins, each containing 10\% of the dataset, and plot the average samples-to-correct for each bin for both \repexp and random sampling. As displayed in~\cref{fig:model_quality_vs_improvement} (right), \repexp matches or improves verifier efficiency across all bins, with the largest improvements on the hardest bins (e.g., the hardest 20\% of questions on \texttt{MATH}). Concretely, on the hardest \texttt{Game-of-24} questions, we find that \repexp with \texttt{Phi-4} provides a 3x improvement in verifier efficiency.\loose

\paragraph{\underline{RF4}: \repexp improves verifier efficiency over standard generation modifications} We now investigate the effect of alternative base sampling strategies that might already induce diversity. Using \texttt{Qwen-2.5-7B-Instruct} on \texttt{MATH}, we change the underlying generation strategy to use one of five different generation settings: vanilla (no changes), low temperature ($\tau = 0.6$), high temperature ($\tau = 1.5$), min-p~\citep{minhturning} ($\tau = 1.5, p = 0.05$), and nucleus sampling~\citep{holtzmancurious} ($\text{top-p} = 0.9$). In \cref{fig:different_candidate_pools}, we find that \repexp improves verifier efficiency in all settings, except for when paired with 
high-temperature sampling. We suspect this is because high-temperature sampling tends to produce less coherent responses, which may look novel in representation space, yet do not necessarily contain correct answers.

\subsection{Extension: Representation-based exploration at the token level}
\label{sec:guided}

While useful in its own right as a testbed for benchmarking the viability of exploration methods, one drawback of the inference-time selection setting is that compute---as measured by $N$, the size of the per-question data pool---may need to be rather large relative to $k$ for selection to yield improvements. As an extension, we conduct a preliminary investigation into algorithms that use elliptic bonuses to guide the autoregressive generation process itself, removing the need to generate such a pool at all.

\arxiv{\begin{wrapfigure}{r}{0.45\linewidth}
    \centering
    \vspace{-0.6cm}
    \includegraphics[width=\linewidth]{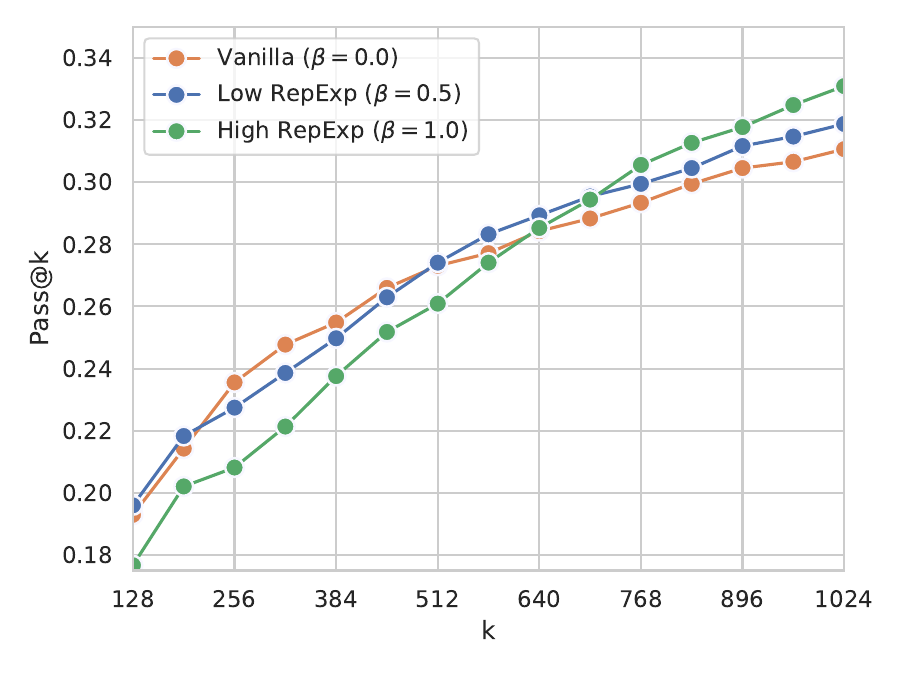}
    \vspace{-0.6cm}
    \caption{\textbf{Representation-based exploration at the token level}, compared to naive autoregressive generation ($\beta=0$) for inference-time exploration. \repexp improves pass@k for large $k$ over naive sampling. 
    }
    \vspace{-0.5cm}
    \label{fig:token_level}
\end{wrapfigure}}

\paragraph{Representation-based exploration for autoregressive generation} To guide sampling for improved diversity, given a budget $k$, we use features from responses $1,\ldots,i-1$ to guide the generation of the $i\text{th}$ response by modifying the logits at every generation step. 
Specifically, consider the $i\text{th}$ generation for a given prompt $x$. At each position $t$ \emph{within} the generation, we perturb the $|V|$-dimensional token-level logit vector as:\loose
\[
    \tilde{\mathbf{z}}^{(i)}(x, y_{<t}) \;=\; \mathbf{z}^{(i)}(x, y_{<t}) + \beta \cdot \mathbf{b}^{(i)}(x, y_{<t}, V), \quad 
\tilde{\mathbf{z}}^{(i)}(x) \in \mathbb{R}^{|V|},
\]
where the bonus $\mathbf{b}^{(i)}(x, y_{<t}, V)$ is a token-level elliptic bonus, defined as:
\[
    \mathbf{b}_j^{(i)}(x, y_{<t}, V) = \sqrt{\tilde{h}_{\theta}(x, y_{<t}, v_j)^\top \Sigma_{(i)}^{-1} \tilde{h}_{\theta}(x, y_{<t}, v_j)},
\]
for $v_j \in V$. Here $\tilde{h}_{\theta}(x, y_{<t}, v_j)$ is a mean-centered Transformer representation for sequence $(y_{<t}, v_j)$; see \cref{sec:appendix_guided} for further details.%

\iclr{\begin{wrapfigure}{r}{0.45\linewidth}
    \centering
    \includegraphics[width=\linewidth]{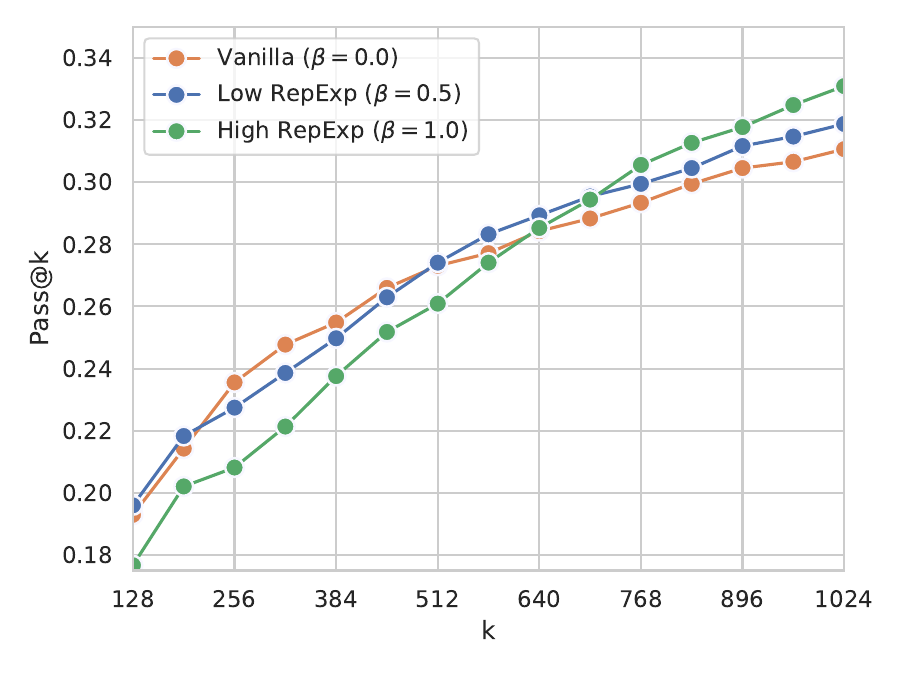}
    \iclr{\vspace{-0.6cm}}
    \caption{\textbf{Representation-based exploration at the token level}, compared to naive autoregressive generation ($\beta=0$) for inference-time exploration. \repexp improves pass@k for large $k$ over naive sampling. 
    }
    \iclr{\vspace{-0.5cm}}
    \arxiv{\vspace{-2cm}}
    \label{fig:token_level}
\end{wrapfigure}}

    \textbf{\mbox{\underline{RF5}: \repexp for autoregressive generation improves solve rate}.}
    In ~\cref{fig:token_level}, we visualize the effect of token-level representation-based exploration for \texttt{Qwen-2.5-7B-Instruct} on the \texttt{MATH} task. We use two values for the bonus parameter $\beta$ (0.5, 1.0) and compare the pass@k to vanilla autoregressive generation for the 200 hardest (but solvable) questions in \texttt{MATH}, as judged by \texttt{GPT-4o-mini}. While token-level exploration tends to solve fewer problems compared to vanilla generation when given a small budget, this trend reverses when the budget exceeds $512 - 640$ (depending on choice for $\beta$).~\cref{fig:token_level_full} further shows that the improvement in solve rate is largest on the hardest questions. These results are encouraging, though further research is required (e.g., on more tasks and models) before one can draw a definitive conclusion. Further, our implementation is not optimized for efficiency and requires at least one additional forward pass per generation step compared to naive sampling.\loose

\arxiv{
\subsection{Application: Protein sequence generation}
\label{subsec:application}
\input{sections/protein}
}

%% file: sections/algorithm_elliptic.tex
\begin{algorithm}[H]
  \caption{\texttt{RepExp}} %
  \label{alg:elliptic_coreset}
  \begin{algorithmic}[1]
    \State\multiline{\textbf{input:} \mbox{Embeddings $\bar{h}_\theta$ (abbrv. $\bar{h}$)},\\
    generations $\cY$ for prompt $x$,\\
    budget $k$, 
    regularization param. $\lambda$.
    }
    \State Initialize $L \gets \crl{y_1}, y_1 \sim \text{Unif}(\cY)$.
    \State Initialize inverse covariance $\Lambda_0 = \lambda^{-1} I_d$.
    \vspace{1pt}
    \For{$t = 1$ to $k-1$}
        \vspace{1pt}
        \State $\Lambda_t \gets \Lambda_{t-1} - \frac{\Lambda_{t-1} \bar{h}(x,y_t)\bar{h}(x,y_t)^\top \Lambda_{t-1}}{1+\bar{h}(x,y_t)^\top \Lambda_{t-1} \bar{h}(x,y_t)}$.
        \vspace{2pt}
        \State $y_{t+1} = \operatorname*{argmax}\limits_{y \in \cY} \bar{h}(x, y)^\top \, \Lambda_t \, \bar{h}(x, y)$.
        \State $L \gets L \cup \crl{y_{t+1}}$.
        
    \EndFor
    \State \textbf{return:} $L$.
  \end{algorithmic}

\end{algorithm}

%% file: sections/protein.tex
\iclr{
\begin{figure}
    \centering
    \includegraphics[width=0.9\linewidth]{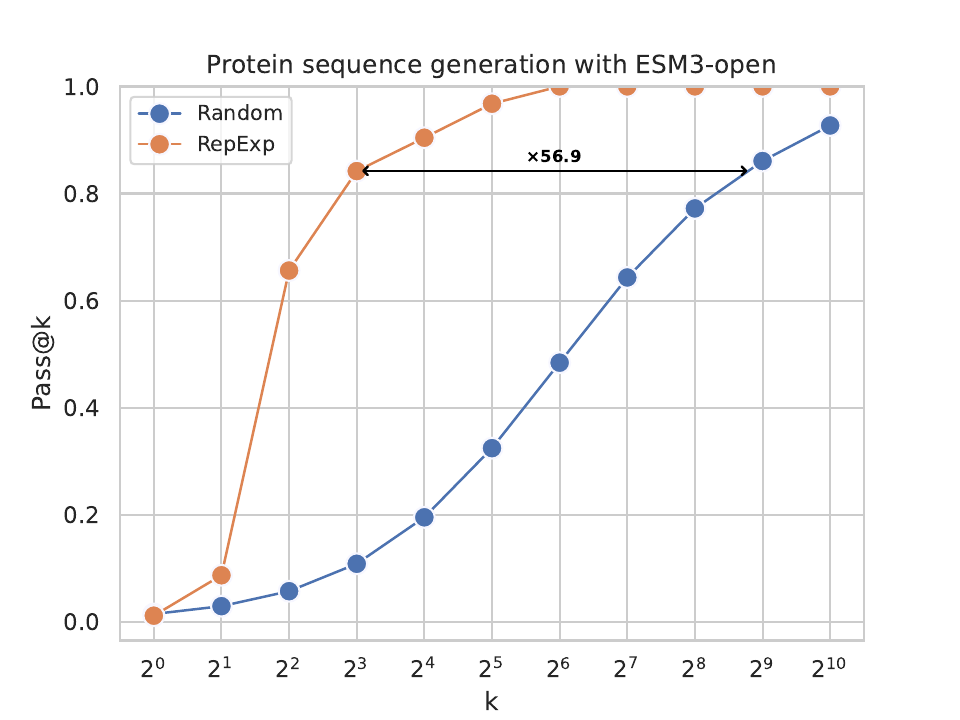}
    \caption{\textbf{Benefits of \repexp for protein sequence generation.} We plot the pass@$k$ curve for random vs. \repexp when performing inference-time selection on protein sequences from \texttt{ESM3-open}. \textbfc{Compared to random sampling, \repexp provides up to 56.9x improvements in verifier efficiency.}
    }
    \label{fig:protein_pass_at_k}
\end{figure}
}
As pointed out in \cref{remark}, we believe inference-time exploration can be useful in its own right for domains where verification is expensive. While an extensive study of this merits further study, we provide some preliminary experiments in the domain of protein sequence generation. In this domain, verification requires lab work, which is much more time consuming than sampling protein sequences from a generative protein model. Following prior work, we use two proxy metrics described below in lieu of expensive wet-lab verification.

\arxiv{\begin{wrapfigure}{r}{0.5\linewidth}
    \centering
    \vspace{-0.7cm}
    \includegraphics[width=\linewidth,trim={0 0 0 0.25cm},clip]{figures/protein_pass_at_k.pdf}
    \vspace{-0.5cm}
    \caption{\textbf{Benefits of \repexp for protein sequence generation.} We plot the pass@$k$ curve for random vs. \repexp when performing inference-time selection on protein sequences from \texttt{ESM3-open}. \textbfc{Compared to random sampling, \repexp provides up to 56.9x improvements in verifier efficiency.}
    }
    \vspace{-0.1cm}
    \label{fig:protein_pass_at_k}
\end{wrapfigure}}

For our experimental setup, we follow the unconditional generation setup in~\citet{esm3} and sample 2048 protein sequences for fully masked sequence prompts ranging in length from 64 to 916 with increments of size 4 (i.e. {64, 68, ..., 916}) using \texttt{ESM3-open}. Then, we perform structure prediction for all sampled sequences using ESMFold~\citep{esm_fold}, which returns pLDDT and pTM scores per sequence. We count a sampled sequence as plausible when pLDDT > 0.8 and pTM > 0.8. To get representations for every sequence, we use \texttt{esmc-600m}, the 600M parameter version of ESM C~\citep{esm3,esm2024cambrian}, and average the last-layer hidden representations. Finally, we perform inference-time selection for every prompt and its 2048 corresponding sampled sequences, using both random selection and \repexp. We find the average samples-to-correct for random to be 240.3 and for \repexp to be 6.7. This corresponds to a 35.9x verifier efficiency improvement. We also plot the corresponding pass@k curves in~\cref{fig:protein_pass_at_k}, for which we find up to 56.9x improvements in verifier efficiency.

We find these results to be very promising and suggestive that our method can be practical in a domain where verification is genuinely expensive.

%% file: sections/rl.tex
Following the methodology in \cref{sec:preliminaries}, we now investigate the use of representation-based exploration to guide the RL post-training process.

\paragraph{Tasks and models} We use \texttt{Qwen-2.5-7b-Instruct} evaluated on \texttt{MATH}, \texttt{GSM8K}, and \texttt{AIME 2024}. Because there are only $30$ questions in \texttt{AIME 2024}, we follow~\citet{yu2025dapo} and use the \texttt{DAPO-Math-17K} dataset for training, leaving \texttt{AIME 2024} for evaluation only. Please refer to~\cref{sec:appendix_rl} for exact details on train, validation, and test splits for all tasks.

\paragraph{Baselines} We compare our method with three baselines:
\textbf{(1) Unlikeliness}~\citep{he2025rewarding} modifies GRPO by scaling the extrinsic rewards by a value inversely related to the likelihood of a generation under the current policy. \textbf{(2) GRPO} is simply an unmodified version of the original GRPO algorithm. \textbf{(3) Base Model} is the original untrained model included as a reference point to see if methods can improve upon it, especially at high values of $k$.\loose

\textbf{Representation-based Exploration (\repexp).} We augment the rewards in GRPO with representation-based bonuses as described in \cref{sec:representation}. Concretely, we add sequence-level elliptic bonuses to the binary extrinsic rewards provided by the verifier: $r_i = R(x,y_i) + b(x, y_i)$ for the $i^{\text{th}}$ rollout $y_i$ of a given prompt $x$. As mentioned in \cref{sec:representation}, we specifically use leverage score-like elliptic bonuses, which allow easier control over their scale as they are bounded in $[0, 1]$.
The covariance matrix $\Sigma$ used to compute bonuses is re-initialized for each batch of RL training. This way, the bonus $b(x, y_i)$ measures the novelty of $y_i$ only with respect to the other rollouts in the batch---previously generated sequences for $x$ are not considered. To better maximize the bonus along all relevant directions in representation space, we draw a new random projection of $\bar{h}_\theta(x, y_i)$ at each optimization step. For further details, please refer to~\cref{sec:appendix_rl}.

\arxiv{
\begin{wrapfigure}{r}{0.45\linewidth}
    \centering
    \vspace{-0.64cm}
    \includegraphics[width=\linewidth]{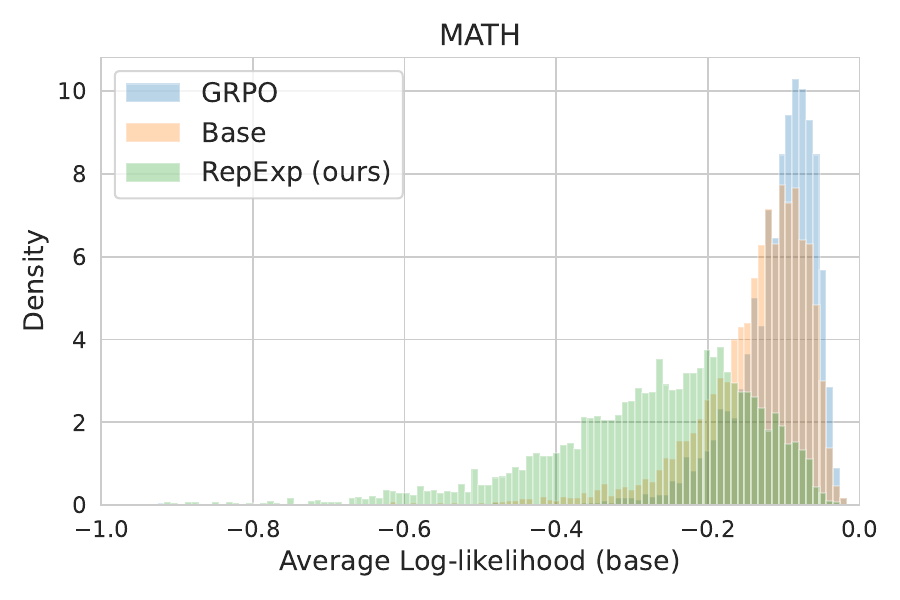}
    \vspace{-0.5cm}
    \caption{\textbf{Anti-sharpening behavior of \repexp.} We plot the histogram of average log-likelihoods evaluated under the base model. While GRPO exhibits sharpening, the responses from \repexp are significantly more novel. See~\cref{sec:appendix_beyond_sharpening} for results on \texttt{GSM8K}.\loose
    }
    \vspace{-0.6cm}
    \label{fig:log_prob_histogram_math}
\end{wrapfigure}
}

\textbf{\underline{RF6}: \repexp improves pass@k.} \cref{fig:rl_results} compares pass@$k$ curves after training for all methods. In line with earlier work, we find that all instantiations of GRPO improve the pass@k for small values of $k$, and that standard GRPO degrades performance relative to the base model for large values of $k$~\citep{yue2025does}. Exploration appears to be an essential part of mitigating the latter effect: policies fit using \repexp preserve or improve pass@$k$ for large values of $k$ with limited reductions for small $k$. This phenomenon is more pronounced for \repexp than for Unlikeliness.

\iclr{
\begin{wrapfigure}{r}{0.45\linewidth}
    \centering
    \vspace{-0.6cm}
    \includegraphics[width=\linewidth]{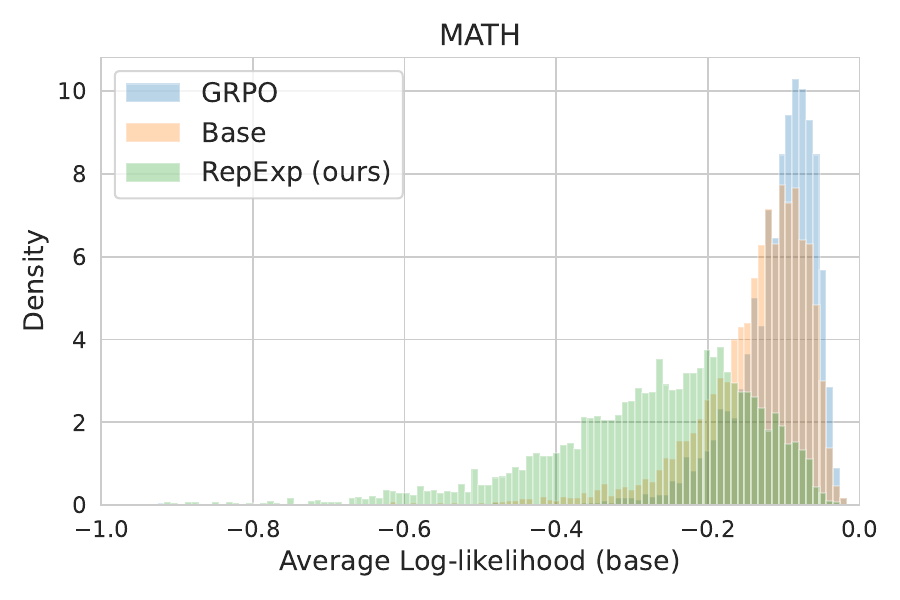}
    \vspace{-0.5cm}
    \caption{\textbf{Anti-sharpening behavior of \repexp.} We plot the histogram of average log-likelihoods evaluated under the base model. While GRPO exhibits sharpening, the responses from \repexp are significantly more novel. See~\cref{sec:appendix_beyond_sharpening} for results on \texttt{GSM8K}.\loose
    }
    \vspace{-0.4cm}
    \label{fig:log_prob_histogram_math}
\end{wrapfigure}
}

\paragraph{\underline{RF7}: \repexp responses look novel} To provide further insight into whether \repexp is able to move beyond merely sharpening the base model, we run an additional experiment inspired by Figure 4 in~\citet{karan2025reasoning}. Specifically, we sample a single response from the base model, the GRPO post-trained model, and the \repexp post-trained model for each question in the full test set of \texttt{MATH}. We then score all responses under the base model in terms of log likelihood. We plot the resulting histogram in~\cref{fig:log_prob_histogram_math}. We find that responses from \repexp tend to be less likely under the base model, indicating that it generates more novel responses and is therefore not merely performing sharpening. In contrast, notice that standard GRPO exhibits sharpening behavior, as demonstrated by the movement of probability mass towards the right with respect to the base model. Taken together with RF6, these results suggest that with the right exploration strategy, we may be able to escape the sharpening regime and discover novel model behaviors.\loose

%% file: sections/conclusion.tex
\paragraph{Related work}

\input{sections/related_body}

\paragraph{Final remarks} Our work shows that deliberate exploration is a viable path toward expanding the reasoning capabilities of language models, offering the possibility of discovering novel behaviors that would be unlikely under naive sampling. While our results show that representation-based diversity is effective at incentivizing exploration, the algorithm design space for exploration techniques is vast, and there is still much to understand regarding how to best use the knowledge encoded in foundation models to guide exploration. Along these lines, natural directions for future work include:
\begin{enumerate}
    \item \emph{Scaling up RL compute,} and combining exploration with other techniques known to improve reasoning behavior in RL post-training, such as prolonged reinforcement learning \citep{liu2025prorl}.\loose
    \item \emph{Exploration for autoregressive generation.} Our results in \cref{sec:guided} show that incentivizing diversity during autoregressive generation is a promising approach to reducing the computational burden of exploration, but much remains to be done in terms of (1) understanding which diversity metrics are most helpful, and (2) optimizing the implementation to close the compute gap.
\item \emph{Beyond verifiable rewards.} How can we deliberately incentivize exploration in domains without verifiable rewards, while simultaneously mitigating reward hacking? 
\end{enumerate}

%% file: sections/related_body.tex
    Several recent works aim to encourage exploration in language models, either by adapting exploration techniques from deep reinforcement learning, or by augmenting PPO or GRPO in ways that are more specialized to language models~\citep{he2025rewarding,cheng2025reasoning,chen2025passk,zhou2025expo,setlur2025e3,liu2025prorl}. Examples of the former include count-based exploration via pseudo-counts~\citep{bai2025online} and at the outcome level~\citep{song2025outcome}, random network distillation~\citep{liu2024sample,gao2025navigate}, and posterior sampling~\citep{dwaracher2024laefficient}. Examples of the latter include rewarding unlikely-but-correct responses~\citep{he2025rewarding}, entropy bonuses~\citep{cheng2025reasoning}, adapting the number of rollouts based on question difficulty~\citep{yang2025depth}, training the model to explore in-context~\citep{setlur2024rewarding}, using a learned classifier to jointly optimize diversity and quality~\citep{li2025jointly}, adding a diversity term based on determinantal point processes to the RL objective~\citep{chen2025enhancing},  and reformulating post-training to more-directly maximize the pass@$N$ objective~\citep{balashankar2024infalign,chow2025inference,chen2025passk,walder2025pass,tangoptimizing}.
    Among these, we experiment with the unlikeliness reward approach of~\citet{he2025rewarding} as a baseline, due to its robust performance and clean implementation. 
    More generally, our work is unique in (1) the specific representation-based objective, and (2) our focus on \emph{inference-time} as a means to validate methods with minimal confounding factors. 
    See~\cref{sec:additional_related} for a detailed overview.

%% file: sections/acknowledgements.tex
We thank Adam Block, Qinghua Liu, and 
Max Simchowitz for valuable feedback on this work. We thank Manan Tomar, Audrey Huang, Spencer Whitehead, Prabhat Nagarajan, and Karthik Narasimhan for support and encouragement throughout the project.

%% file: sections/reproducibility.tex
To ensure reproducibility of our results, we have listed all relevant details (hyperparameters, experiment resources, etc.) for the inference-time experiments in~\cref{sec:appendix_inference_time}, and those for RL post-training in~\cref{sec:appendix_rl}. In addition, the code for all our experiments and plots can be found at \url{https://rep-exp.github.io}. We used Weights \& Biases for experiment tracking and visualizations to develop
insights for this paper.

%% file: sections/related.tex
\paragraph{Exploration at test time} Test-time alignment techniques for language models are an active area of research with many complementary threads \citep{khanov2024args,chen2024pad,shi2024decoding,liu2024decoding,jinnai2024regularized,shi2024crucial}, but exploration has not typically been the focus of this line of work.\loose

Most closely related to our work, \citet{setlur2025e3}, propose a test-time exploration approach based on the idea of \emph{learning to explore in-context}. They propose to encourage exploration within a long chain of thought by
training the LLM to chain operations such as generation, verification, and refinement together in search of a solution. This is somewhat complementary to our inference-time exploration framework, which aims to improve diversity across parallel generations once the model is fixed; these techniques could potentially be combined.\loose

Also related, \citet{xu2025provably} consider the problem of learning from language (non-verifiable) feedback, and propose an iterative prompting approach to enable exploration at test time; their work focuses on simpler exploration domains, but with more difficult, implicit feedback.

\paragraph{Exploration in RL post-training} Exploration in RL post-training for reasoning is a growing area of research, motivated by the observation that standard techniques tend to simply sharpen responses already covered by the base model \citep{yue2025does,gandhi2025cognitive,wu2025invisible}. A number of recent works, discussed below, aim to improve diversity and expand the reasoning frontier by incorporating bonuses into the GRPO objective or by otherwise augmenting it. Briefly, our work is unique in terms of (1) the specific representation-based diversity objective we focus on, and (2) our focus on \emph{inference-time exploration} as a means to validate diversity metrics before applying them to post-training. \loose

\citet{he2025rewarding} introduce an \emph{unlikeliness reward} to GRPO, which reweights the reward by ranking generations according to their unlikeliness under the sampling policy. Unlikeliness reward is a form of diversity metric, similar to our representation-based diversity metrics. \citet{cheng2025reasoning} observe that high-entropy (high uncertainty) tokens in the model's output often correspond to critical reasoning steps, and augment the GRPO objective with entropy bonuses to encourage exploration at these high-entropy steps. Entropy can be seen as another form of diversity metric in our setup. Another option is to \emph{learn} the diversity metric as in~\citet{li2025jointly}, who use a learned classifier to determine whether a pair of responses is semantically equivalent. They then use the classifier score to scale the reward in GRPO to joinly optimize quality and diversity. Our work, in contrast, does not require training any auxiliary models for computing diversity.

Various works \citep{balashankar2024infalign,chow2025inference,chen2025passk,walder2025pass,tangoptimizing} formulate the problem of directly post-training to maximize the \passatk objective, deriving approximate gradient estimators and using them for policy optimization. As discussed in \citet{chow2025inference,chen2025passk}, these gradient estimators \emph{implicitly} encourage exploration, since they allow the model to distribute probability mass across a more diverse range of responses when it is uncertain about the correct answer. Our work instead focuses on using the language model representations to \emph{deliberately} incentivize novel behaviors, including in a novel inference-time setting. \loose

\citet{zhou2025expo} consider a setting where ground truth answers are available (as opposed to just rewards), and propose to encourage exploration by prompting the model to generate self-explanations for the ground truth answers and incorporating this as an SFT term in the GRPO loss. Unlike our method, their approach does not directly optimize for diversity, and cannot be used in settings where ground truth answers are unavailable (e.g. coding).\loose

\citet{yang2025depth} investigate the role of ``breadth" (batch size) and ``depth'' (number of rollouts) in RLVR. They show that increasing breadth through full batch updates and increasing depth through more rollouts for harder questions has complementary benefits and overall improves pass@1 and pass@k performance. We view this as orthogonal and potentially complementary to our work.

\citet{lanchantin2025diverse} introduce diverse preference optimization (DivPO), an alignment method to optimize for both quality and diversity. In contrast to our work, their method is designed for the RLHF setting and applied to non-reasoning tasks (e.g. creative writing).

Concurrent work of \citet{song2025outcome} adapts tabular UCB-style bonuses to language model post-training with GRPO, but their approach---unlike representation-based exploration---is only suitable for domains with a small, discrete set of possible outcomes. Other concurrent work of~\citet{chen2025enhancing} optimizes a diversity term along with the rewards, where the diversity term captures the volume of the responses in representation space by computing the determinant of the gram matrix. While related, our method instead adds a leverage-based score to the reward of each individual response.

Lastly, \citet{liu2025prorl} take a complementary approach and aim to incentivize reasoning beyond the base model through (1) prolonged RL training (increasing the overall amount of training steps), and (2) periodically resetting the reference model; they show that this can increase \passatk performance beyond the base model in a variety of reasoning tasks. This approach is complementary, and could likely be combined with our techniques.\loose

\paragraph{Representation-based diversity}
Our findings regarding benefits of inference-time exploration with representation-based diversity parallel the findings of \citet{ivison2025large}, who evaluated the effectiveness of various data selection schemes for instruction tuning, and found a similar representation-based scheme to be the most effective when normalized for compute. In addition, there is a long line of work relying on representation-based exploration for RL through elliptic bonuses in non-LLM settings~\citep{agarwal2020pc,agarwal2020flambe,henaff2022exploration,ash2021anti}.

\paragraph{Adapting exploration techniques from deep reinforcement learning}
Various papers have adapted exploration techniques from deep learning to language models, including \citet{bai2025online} (count-based exploration), \citet{gao2025navigate} (random network distillation), and \citet{liu2024sample,dwaracher2024laefficient} (posterior sampling).\footnote{See also \citet{arumugam2025toward}, which uses a pre-trained model to simulate posterior sampling in-context for multi-turn sequential decision making tasks.} These works show initial promise in terms of sample complexity benefits, but their potential to explore beyond the base model in reasoning domains has not been evaluated to our knowledge. In addition, these methods require additional learning machinery (e.g., auxiliary networks), which introduces significant complexity when scaling to language models.\loose

\paragraph{Theoretical analysis of language model exploration}
On the theoretical side, our work draws on \citet{foster2025good}, who prove that test-time exploration with representation-based diversity has provable computational benefits in a simplified linear setting. Our \repexp algorithm for test-time exploration can be viewed as a simplified, practical adaptation of their theoretical algorithm.\loose

Other theoretical works on exploration with language models include the \xpo algorithm of \citet{xie2024exploratory} and related algorithms by \citet{cen2024value,zhang2024self},\footnote{\citet{cen2024value,zhang2024self} concurrently proposed
  similar algorithms to \xpo, but did not provide non-trivial
  theoretical guarantees (e.g., guarantees that indicate benefits over
  purely passive exploration).} which augment the Online DPO objective with exploration bonuses inspired by the optimism principle. To our knowledge, these techniques have only been evaluated on RLHF tasks, and \citet{foster2025good} show that there may be computational barriers to implementing them in a way that is faithful to the theoretical guarantees.

%% file: sections/details_for_coreset.tex
\paragraph{Hyperparameters} In ~\cref{alg:elliptic_coreset}, we set $\lambda = 1.0$. For all models and tasks, we perform a sparse projection from the respective model hidden dimension to $d = 512$.

\paragraph{Preprocessing} In~\cref{alg:elliptic_coreset}, after obtaining the representations $\bar{h}_\theta$ for every generation $y$ for a fixed prompt $x$, we sparse project all representations down and then mean-center where the mean is taken across the response-level representations. 

\paragraph{Datasets} Below we provide a brief overview of all datasets along with relevant numerical details. Note that "vanilla" sampling settings refer to $\tau = 1.0$, $\text{top-p} = 1.0$, and $\text{min-p} = 0.0$. We also do not use top-k sampling in any of the coreset experiments. Finally, we only use the test split of every dataset for all our inference-time experiments, unless specified otherwise.
\begin{itemize}
    \item \texttt{MATH}. This dataset contains 12.5k problems from high school math competitions, split into 7.5k training examples and 5k test examples. For each question in the test split, we generate $6400$ responses using vanilla settings and set the maximum response length per generation to $512$ tokens.
    
    \item \texttt{GSM8K}. This dataset contains 8.79k grade school math word problems, split into 7.47k training examples and 1.32k test examples. For each question in the test split, we generate $6400$ responses using vanilla settings and set the maximum response length per generation to $512$ tokens.
    
    \item \texttt{MBPP+}. This dataset contains 378 basic Python programming problems that are a curated subset of the full \texttt{MBPP}~\citep{austin2021program} dataset with more test cases. Since the dataset does not come with any train or test splits, we use the full set of questions for our experiments. For each problem, we generate $6400$ responses using vanilla settings, except that we set $\text{top-p}=0.95$. We set the maximum response length per generation to $768$ tokens.
    
    \item \texttt{Game of 24}. This dataset contains 1.36k questions that specify four integers that need to be combined using basic arithmetic operations $(+, -, x, /)$ to equal $24$. For each question, we generate $6400$ responses using vanilla settings and set the maximum response length per generation to $512$ tokens.  We use the version available at \url{https://huggingface.co/datasets/nlile/24-game}.
    
    \item \texttt{AIME 2025}. This dataset contains the $30$ problems taken directly from the 2025 edition of the American Invitational Mathematics Examination (AIME). For each question, we generate $8192$ responses using vanilla settings and set the maximum response length per generation to $8192$ tokens.

\end{itemize}

\paragraph{Experiment resources} We used vLLM~\citep{kwon2023efficient} on $1-2$ (depending on the size of the model) NVIDIA A100 40GB GPUs per model-task pair to generate the data pools for all questions in the dataset.

\paragraph{Estimating samples-to-correct} For random sampling, we estimate the average number of samples to take (without replacement) from the data pool to find the first correct one as:
\[
    \text{samples-to-correct} = \frac{N + 1}{c + 1},
\]

where $N$ indicates the size of the data pool and $c$ indicates the number of correct samples in the pool. Please refer to~\citet{rohatgi2015introduction} for a proof. For the representation-based exploration algorithm described in~\cref{alg:elliptic_coreset}, we perform $5$ trials per question where we record the samples-to-correct for each and take their average.

\paragraph{Estimating pass@k} To compute the pass@k values plotted in Figure~\ref{fig:different_candidate_pools}, we follow~\citet{chen2021evaluating} and use the following unbiased estimator:
\[
    \text{pass@k} = \mathbb{E}_{\mathcal{D}}\left[  1 - \frac{\binom{n - c}{k}}{\binom{n}{k}}\right] = \frac{1}{|\mathcal{D}|} \sum_{i = 1}^{|\mathcal{D}|} \left[  1 - \frac{\binom{n - c}{k}}{\binom{n}{k}}\right],
\]

where $\mathcal{D}$ indicates the dataset and $|\mathcal{D}|$ indicates its size.

%% file: sections/details_for_token_level.tex
\paragraph{Hyperparameters} Similarly to \cref{sec:coresets}, we initialize $\Sigma_{(0)}^{-1} = \lambda^{-1} I_d$. We set $\lambda = 0.1$. For all values of $\beta$, we use $\text{top-p} = 0.95$ and $\text{top-k} = 128$. At every time step $t$, we use a batch size of 64 to compute $h_{\theta}(x, y_{<t}, v_j)$ for all $v_j \in V$ where $v_j > -\infty$ (note that a logit $v_j$ is set to $-\infty$ if it gets filtered out by either the top-p or top-k filters mentioned earlier).

\paragraph{Computational expense} Note that computing the bonus $\mathbf{b}(x, y_{<t}, V)$ requires one forward pass through the model per token in the vocabulary at every time step $t$. While these can be batched together since all tokens share the same prefix, this is still prohibitively time and memory intensive. To mitigate this, we combine this method with nucleus and top-k sampling such that the bonus will only need to be computed for at most $k \ll V$ tokens.

\paragraph{Dataset construction} Due to the large computational cost of experiments, we only focus on \texttt{MATH} using \texttt{Qwen-2.5-7b-Instruct}. In addition, we do not evaluate on the full test split of \texttt{MATH}, but instead use a subset consisting of the $200$ hardest questions as ranked by \texttt{GPT-4o mini}. Specifically, we sampled $1024$ responses for each question in the \texttt{MATH} test split using \texttt{GPT-4o mini} to estimate the per-question pass@1. We threw out all questions for which the pass@1 was $0$ (indicating not a single response was correct among all $1024$), sorted the remaining questions, and kept the $200$ questions with the lowest pass@1.

\paragraph{Estimating samples-to-correct} We collected up to $1024$ generations per question (stopping early once the correct answer was found). Since the resulting samples-to-correct value found can have high variance due to the inherent randomness in the generation process, we repeat this process $5$ times with a different seed every time. This then results in a total of $200 \times 5 = 1000$ data points, minus a few data points that didn't finish running in time, for an effective total of $985$ data points used in~\cref{fig:token_level}.

\paragraph{Experiment resources} Experiments were run on a combination of NVIDIA A100 40GB, NVIDIA A100 80GB, and NVIDIA H100 80GB GPUs. Every run indicates one seed for a fixed question and performs up to $1024$ generations. We used one GPU (one of the several mentioned earlier) per run, adjusting the forward batch size to compute the elliptical bonus down from $64$ to $32$ when using the 40GB GPUs.

\paragraph{Computing the inverse covariance} The inverse covariance matrix $\Sigma^{-1}_{(i)}$ includes all \emph{mean-centered} hidden representations for all generated tokens in all $i-1$ complete sequences generated so far for a fixed prompt $x$. Note that the mean $\mu^{(i)}$ of the raw hidden representations $h_{\theta}$ is computed as:
\[\
    \mu^{(i)} = \frac{1}{H} \sum_{j = 0}^{i - 1}\sum_{t = 1}^{T_j} h_{\theta}(x, y_{<t}^{(j)}), \quad H = \sum_{j = 0}^{i - 1} T_j,
\]
where $T_j$ indicates the length (in number of tokens) of response $j$. Because of this, the mean changes after every generation $i$. This means some care is required when computing the inverse covariance matrix $\Sigma^{-1}_{(i)}$ with mean-centered hidden representations. To account for this, we will separately keep track of the inverse covariance matrix $\tilde{\Sigma}^{-1}_{(i)}$ with \emph{non}-mean-centered hidden representations as well as the mean $\mu^{(i)}$. Then, we compute the mean-centered inverse covariance matrix $\Sigma^{-1}_{(i)}$ as

\[
    \Sigma^{-1}_{(i)} = \tilde{\Sigma}_{(i)}^{-1} - \left(\frac{\tilde{\Sigma}_{(i)}^{-1}\mu^{(i)}\mu_{(i)}^T\tilde{\Sigma}_{(i)}^{-1}}{-1/H + \mu_{(i)}^T\tilde{\Sigma}_{(i)}^{-1}\mu^{(i)}}\right).
\]

This result is immediate from Proposition~\ref{prop:rank_one}.

\arxiv{
\begin{figure}[t]
    \centering
    \includegraphics[width=\linewidth]{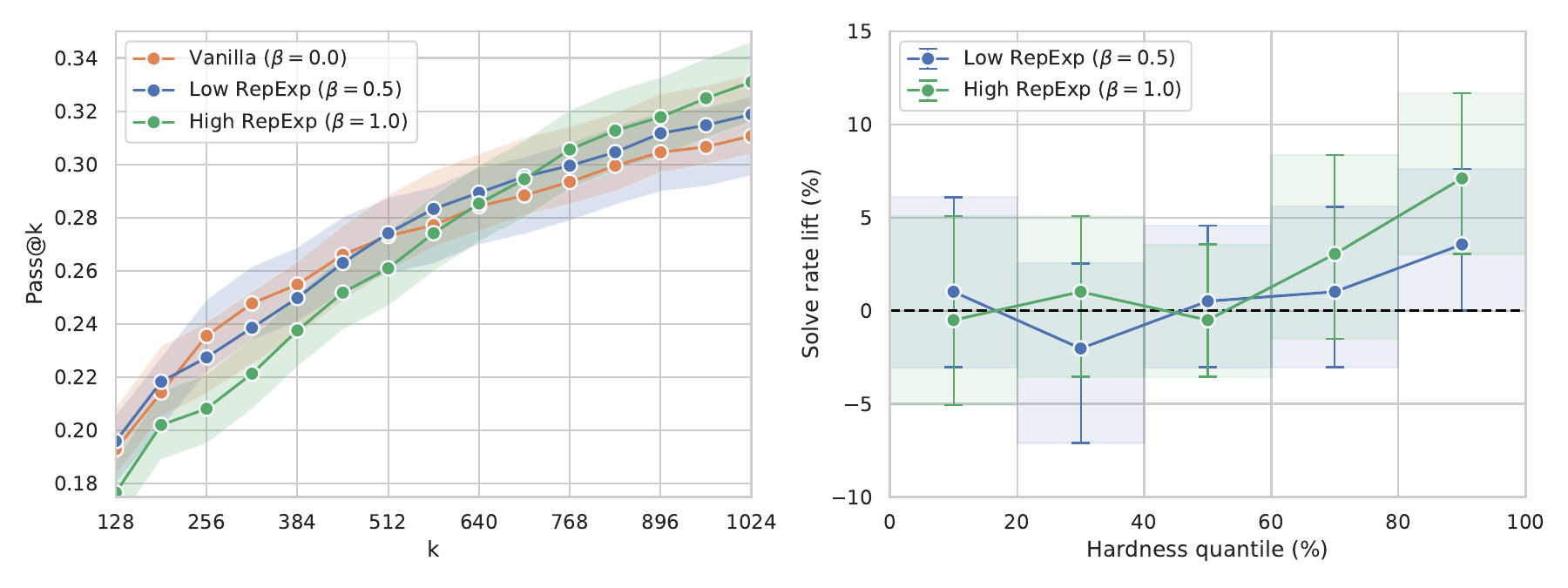}
    \vspace{-0.5cm}
    \caption{\textbf{Representation-based exploration at the token level}, compared to naive autoregressive generation ($\beta=0$) for inference-time exploration. \textbf{\textit{(Left)}} We added shaded areas indicating one standard error to the left side of~\cref{fig:token_level}. \textbf{\textit{(Right)}} Compared to the right side of~\cref{fig:token_level}, we also added $\beta = 0.5$}.
    \vspace{-0.5cm}\label{fig:token_level_full}
\end{figure}
}

\begin{proposition}
    \label{prop:rank_one}
    Given the current generation step $i$, the non-mean-centered inverse data
    covariance matrix $\tilde{\Sigma}^{-1}_{(i)}$ and the current mean $\mu^{(i)}$,
    the mean-centered inverse data covariance matrix $\Sigma^{-1}_{(i)}$ is given
    as
    \begin{equation*}
      \Sigma^{-1}_{(i)} = \tilde{\Sigma}_{(i)}^{-1} - \left(\frac{\tilde{\Sigma}_{(i)}^{-1}\mu^{(i)}\mu_{(i)}^T\tilde{\Sigma}_{(i)}^{-1}}{-1/H + \mu_{(i)}^T\tilde{\Sigma}_{(i)}^{-1}\mu^{(i)}}\right), \quad H = \sum_{j = 0}^{i - 1} T_j.
    \end{equation*}
    Here, $T_j$ indicates the length (in number of tokens) of response $j$.
  \end{proposition}
  \begin{proof}[\pfref{prop:rank_one}]We can write
      \begin{align*}
          \Sigma^{(i)} &= \sum_{j = 0}^{i - 1}\sum_{t = 1}^{T_j} (h^{(j)}_t - \mu^{(i)}_t)(h_t^{(j)} - \mu^{(i)}_t)^T \\
          &= \sum_{j = 0}^{i - 1}\sum_{t = 1}^{T_j} h^{(j)}_t (h^{(j)}_t)^T - \sum_{j = 0}^{i - 1}\sum_{t = 1}^{T_j} h^{(j)}_t(\mu^{(i)})^T - \sum_{j = 0}^{i - 1}\sum_{t = 1}^{T_j} \mu^{(i)} (h_t^{(j)})^T + \sum_{j = 0}^{i - 1}\sum_{t = 1}^{T_j} \mu^{(i)} (\mu^{(i)})^T \\
          &= \tilde{\Sigma}^{(i)} - \left(\sum_{j = 0}^{i - 1}\sum_{t = 1}^{T_j} h_t^{(j)}\right)(\mu^{(i)})^T - \mu^{(i)} \left(\sum_{j = 0}^{i - 1}\sum_{t = 1}^{T_j} (h_t^{(j)})^T\right) + \sum_{j = 0}^{i - 1}\sum_{t = 1}^{T_j}\mu^{(i)}(\mu^{(i)})^T \\
          &= \tilde{\Sigma}^{(i)} - H\mu^{(i)}(\mu^{(i)})^T - \mu^{(i)}H(\mu^{(i)})^T + H\mu^{(i)}(\mu^{(i)})^T \\
          &= \tilde{\Sigma}^{(i)} - 2H\mu^{(i)}(\mu^{(i)})^T + H\mu^{(i)}(\mu^{(i)})^T \\
          &= \tilde{\Sigma}^{(i)} - H\mu^{(i)}(\mu^{(i)})^T.
      \end{align*}
      Inverting $\Sigma^{(i)}$, we conclude that
      \begin{align*}
          \Sigma_{(i)}^{-1} &= \left(\tilde{\Sigma}^{(i)} - H\mu^{(i)}\mu_{(i)}^T\right)^{-1} \\
          &= \tilde{\Sigma}_{(i)}^{-1} - \left(\frac{\tilde{\Sigma}_{(i)}^{-1}\mu_{(i)}\mu_{(i)}^T\tilde{\Sigma}_{(i)}^{-1}}{-1/H + \mu_{(i)}^T\tilde{\Sigma}_{(i)}^{-1}\mu^{(i)}}\right),
      \end{align*}
      where we used the Woodbury matrix identity lemma\footnote{\url{https://en.wikipedia.org/wiki/Woodbury_matrix_identity}} in the last step with $U = \mu^{(i)}$, $C = -H$, and $V = \mu_{(i)}^T$.
  \end{proof}

We exclude representations from the current sequence $i$ in $\Sigma^{(i)}$ to keep the generation from veering off topic, and update $\Sigma^{-1}_{(i)}$ after each generation $i$ using $T_i$ consecutive applications of the Sherman-Morrison update in \cref{alg:elliptic_coreset}, one for each hidden representation of the output tokens. Finally, we found it necessary for numerical stability to perform all covariance related computations in double precision.

\paragraph{Additional plots} In~\cref{fig:token_level_full}, we provide a revised version of~\cref{fig:token_level} where we add shaded areas indicating one standard error to the left plot, and we additionally add $\beta = 0.5$ to the right plot.

\iclr{
\begin{figure}[t]
    \centering
    \includegraphics[width=\linewidth]{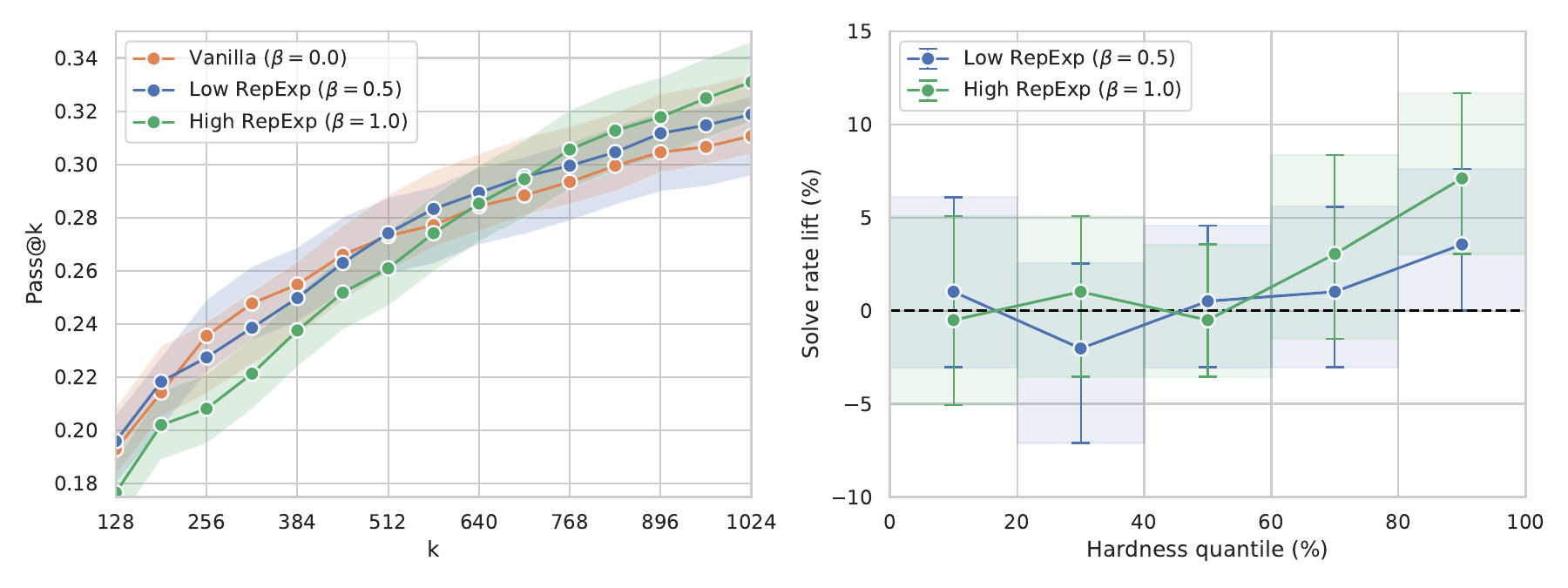}
    \vspace{-0.5cm}
    \caption{\textbf{Representation-based exploration at the token level}, compared to naive autoregressive generation ($\beta=0$) for inference-time exploration. \textbf{\textit{(Left)}} We added shaded areas indicating one standard error to~\cref{fig:token_level}. \textbf{\textit{(Right)}} When binning the questions by hardness (judged by samples-to-correct for \texttt{GPT-4o-mini}), solve rate improves the most on the hardest questions. Error bars indicate 95\% paired bootstrap CIs.}
    \label{fig:token_level_full}
\end{figure}
}

%% file: sections/details_for_rl.tex
\paragraph{Hyperparameters} We use verl for training~\citep{sheng2024hybridflow}, and  provide a full overview of all common hyperparameters in~\cref{tab:rl_hyperparameters}, all hyperparameters specific to unlikeliness in~\cref{tab:unlikeliness_hyperparameters}, and all hyperpameters specific to \repexp in~\cref{tab:elliptical_hyperparameters}. Note that for \texttt{AIME 2024}, we adjusted the maximum prompt length to 2048, the maximum response length to 8192, the train batch size to 512, the ppo mini batch size to 128, and the ppo micro batch size per gpu to 8. Also, for \repexp on \texttt{AIME 2024}, we increased the sparse projection dimension from 32 to 128.

\paragraph{Algorithm details} We note that we mean-center the representations $\bar{h}_\theta$ that are used to compute the elliptic bonuses as described in~\cref{sec:representation}, where the mean is taken over all the response-level representations of the current group of rollouts for a fixed prompt $x$. In addition, we do not add a bonus for questions where \emph{all} rollouts in the batch are incorrect, as we found this to empirically hurt performance.

\paragraph{Dataset splits} For \texttt{MATH}, we use the original 7.5k train split for training, \texttt{MATH-500}~\citep{lightman2023lets} for validation, and the 4.5k (originally 5k minus the problems in \texttt{MATH-500})) test split for testing. For GSM8K, we use the original 8.79k train split for training, the first $512$ examples from the test split for validation, and the remaining $807$ examples from the test split for testing. Finally for \texttt{AIME 2024}, we use $4096$ examples randomly chosen from the full \texttt{DAPO-Math-17K} dataset for training and use the full \texttt{AIME 2024} dataset both for validation and testing.

\begin{figure}[t]
    \centering
        \includegraphics[width=\linewidth]{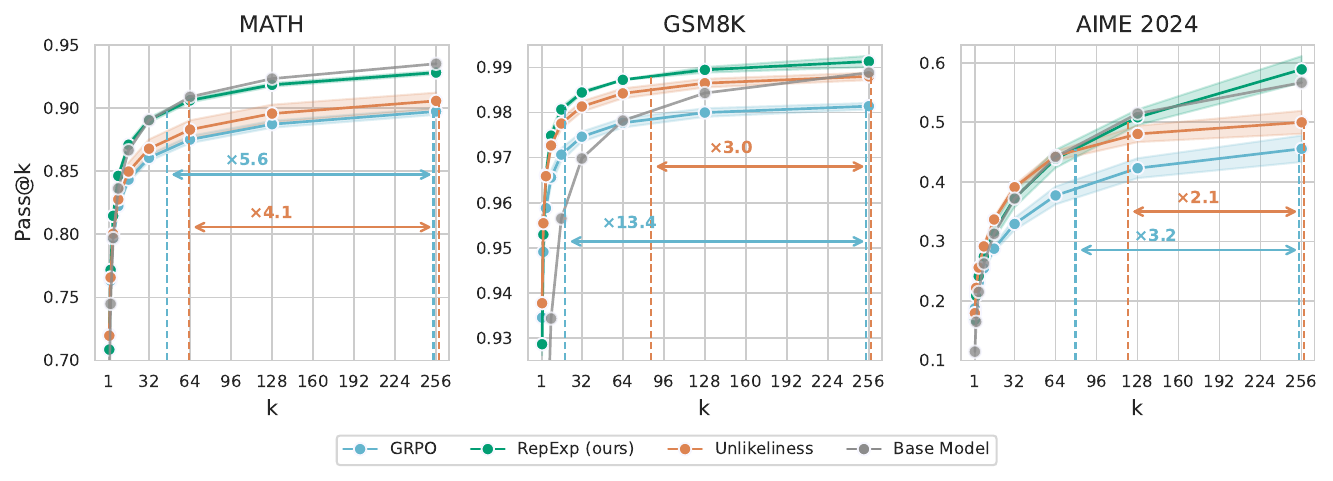}
    \vspace{-0.5cm}
    \caption{\textbf{Pass@k for RL post-training with exploration.} We plot~\cref{fig:rl_results} on linear axes to provide an additional perspective.} 
    \label{fig:rl_results_linear}
\end{figure}

\begin{table}[t]
\centering
\begin{tabular}{@{}cc@{}}
\begin{minipage}{0.55\linewidth}
    \centering
    \caption{\textbf{Common RL hyperparameters.}}
    \label{tab:rl_hyperparameters}
    \begin{tabular}{ll}
        \toprule
        \textbf{Hyperparameter} & \textbf{Value} \\
        \midrule
        Learning rate & $1\text{e}{-6}$ \\
        Maximum prompt length & 1024 \\
        Maximum response length (tokens) & 1024 \\
        Train batch size & 1024 \\
        PPO mini batch size & 256 \\
        PPO micro batch size per gpu & 16 \\
        PPO epochs & 1 \\
        KL loss coefficient & 0.0 \\
        GRPO group size (rollouts) & 8 \\
        Entropy coefficient & 0.0 \\
        Log prob micro batch size per gpu & 16 \\
        Tensor model parallel size & 2 \\
        Number of GPUs & 8 \\
        Validation frequency & 20 \\
        \bottomrule
    \end{tabular}
\end{minipage}
&
\begin{minipage}{0.45\linewidth}
    \centering
    \caption{\textbf{Unlikeliness hyperparameters.}}
    \label{tab:unlikeliness_hyperparameters}
    \begin{tabular}{lc}
        \toprule
        \textbf{Hyperparameter} & \textbf{Value} \\
        \midrule
        $\beta_{\text{rank}}$ & 0.25 \\
        No bonus if all rollouts correct & True \\
        \bottomrule
    \end{tabular}

    \vspace{0.6cm}

    \caption{\textbf{\repexp hyperparameters.}}
    \label{tab:elliptical_hyperparameters}
    \begin{tabular}{lc}
        \toprule
        \textbf{Hyperparameter} & \textbf{Score} \\
        \midrule
        $\beta$ & 0.01 \\
        Sparse projection dimension  & 32 \\
        No bonus if all rollouts incorrect & True \\
        \bottomrule
    \end{tabular}
\end{minipage}
\end{tabular}
\end{table}

\paragraph{Modifications to unlikeliness baseline} The original unlikeliness method from~\citet{he2025rewarding} combines the reward modification described in~\cref{sec:rl} along with several other modifications to the underlying GRPO mechanics:
\begin{itemize}
    \item They only include samples for which the resulting rollouts have nonzero advantages in the batch sent for training. Specifically, questions where either none of the rollouts are correct or all of the rollouts are correct are thrown out. To ensure batch sizes stay roughly equal, the authors implement a buffer mechanism that collects samples until the buffer reaches a target batch size.

    \item They increase the number of ppo epochs from $1$ to $2$ as they find this helps further increase the pass@k.

    \item They use a high KL penalty coefficient of 0.1 as they find this helps prevent the pass@k from decreasing.
\end{itemize}
Since our primary aim in this section is to isolate and compare the exploration mechanisms of our method with others, we leave out the additional changes described above when running the unlikeliness baseline. Furthermore, this allows us to use the exact same underlying GRPO hyperparameters for all methods, making the comparison much more clean.

\paragraph{Checkpoint picking} For each seed per method, we pick the checkpoint during training that achieves the highest pass@1 on the respective task's validation set and use the resulting checkpoint for evaluation.

\paragraph{Evaluation} We evaluate each final checkpoint (picked in the way described earlier) on the test split of each respective task by sampling 256 responses per question using vanilla sampling parameters ($\tau = 1.0$, $\text{top-p} = 1.0$). We then estimate the pass@k exactly as described in~\cref{sec:appendix_coresets}. We run $3$ seeds for all methods on all tasks and average the resulting pass@k curves. In addition, we provide an alternative view of~\cref{fig:rl_results} in~\cref{fig:rl_results_linear}.

\paragraph{Experiment resources} We run all methods on all tasks using 8 NVIDIA H100 80GB GPUs per seed per method for $1$ day.

\paragraph{Computational overhead} On \texttt{MATH}, we estimated \repexp to reach about $0.73$x the steps/hour throughput of GRPO. In other words, \repexp requires about $1.37$x higher wall-clock time per step. Note that this overhead comes from an extra forward pass that's required to compute the hidden representations of the sampled responses at each iteration of the algorithm. However, if either (1) one were to use a KL constraint (which we do not in our experiments) or (2) one were to use the hidden representations from $\pi_t$ (the current policy iterate) instead of $\pi_{\mathrm{ref}}$ (the base model), then the extra forward pass would not be needed and \repexp would have the same throughput as GRPO.

%% file: sections/beyond_sharpening.tex
\begin{figure}
    \centering
    \includegraphics[width=\linewidth]{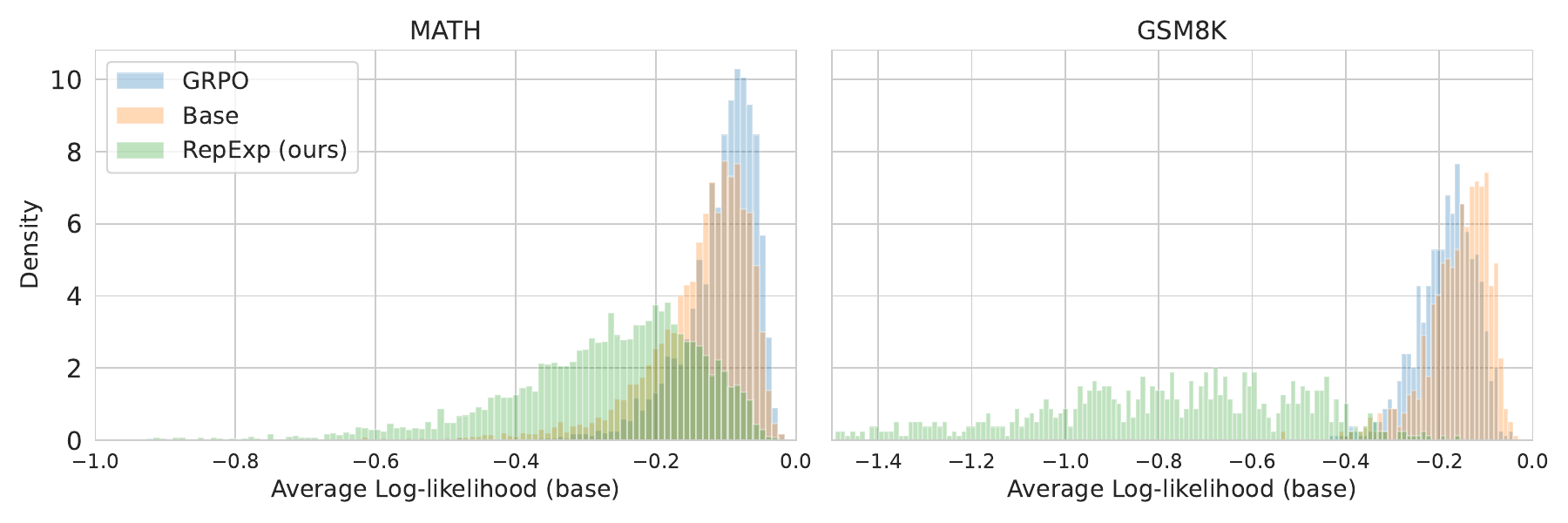}
    \caption{\textbf{Anti-sharpening behavior of \repexp.} We plot the histogram of average log-likelihoods evaluated under the base model for test responses sampled from base, GRPO, and \repexp. While the response likelihoods from GRPO either roughly stay the same or increase, the responses from \repexp are significantly more novel as indicated by the low log-likelihoods under base.
    }
    \label{fig:log_prob_histogram}
\end{figure}
In~\cref{fig:log_prob_histogram}, we provide results for both \texttt{MATH} and \texttt{GSM8K}. We find that on \texttt{GSM8K} the shift toward lower likelihood responses from \repexp as evaluated under the base model is even more dramatic.